\documentclass[twoside,11pt]{article}
\usepackage{jair, theapa, rawfonts}
\jairheading{69}{2020}{923--961}{12/2019}{11/2020}
\ShortHeadings{Qualitative Numerical Planning: Reductions and Complexity}{Bonet \& Geffner}
\firstpageno{923}

\usepackage{theapa}
\usepackage{url} %Required
\usepackage[margin=1in]{geometry}

\usepackage{amsmath,amsthm,amssymb,mathtools}
\usepackage{booktabs}
\usepackage{enumerate}
\usepackage{xspace}

\usepackage[figure,vlined]{algorithm2e}
\usepackage{tikz}
\usetikzlibrary{arrows.meta, calc, positioning, backgrounds}
\tikzset{%
  diagonal fill/.style 2 args={fill=#2, path picture={
    \fill[#1, sharp corners] (path picture bounding box.south west) -|
                             (path picture bounding box.north east) -- cycle;}},
  reversed diagonal fill/.style 2 args={fill=#2, path picture={
    \fill[#1, sharp corners] (path picture bounding box.north west) |-
                             (path picture bounding box.south east) -- cycle;}}
}

\newcommand{\Omit}[1]{}
\newcommand{\tup}[1]{\langle #1 \rangle}

\newcommand{\citeay}[1]{\shortciteauthor{#1}\ \citeyear{#1}}

\newtheorem{definition}{Definition}
\newtheorem{theorem}[definition]{Theorem}

\newenvironment{example}{\bigskip\noindent\textbf{Example.}}{\qed}
\newtheorem*{claim*}{Claim}

\newcommand{\C}{\mathcal{C}}

\newcommand{\abst}[2]{\tup{#1; #2}}

\newcommand{\EQ}[1]{#1{\,=\,}0}
\newcommand{\GT}[1]{#1{\,>\,}0}
\newcommand{\GE}[1]{#1{\,\geq\,}0}
\newcommand{\DEC}[1]{#1\mminus}
\newcommand{\INC}[1]{#1\pplus}

\renewcommand{\S}{\mathcal{S}}
\newcommand{\Q}{\mathcal{Q}}
\newcommand{\R}{\mathcal{R}}
\newcommand{\G}{\mathcal{G}}

\newcommand{\pplus}{\hspace{-.05em}\raisebox{.15ex}{\footnotesize$\uparrow$}}
\newcommand{\mminus}{\hspace{-.05em}\raisebox{.15ex}{\footnotesize$\downarrow$}}

\newcommand{\Eff}{\mathit{Eff}\xspace}
\newcommand{\sieve}{{\sc Sieve}\xspace}

\newcommand{\qnptofond}{\texttt{qnp2fond}\xspace}
\newcommand{\Qnptofond}{\texttt{Qnp2fond}\xspace}

\title{Qualitative Numerical Planning: Reductions and Complexity}
\author{\name Blai Bonet \email bonet@usb.ve \\ %\thanks{This work was done while at sabbatical leave at Universidad Carlos III de Madrid, Spain, under a C\'atedra de Excelencia UC3M-Santander Award.} \email bonet@usb.ve \\
        \addr Universidad Sim\'on Bol\'{\i}var \\
        Caracas, Venezuela
        \AND
        \name Hector Geffner \email hector.geffner@upf.edu \\
        \addr ICREA \& Universitat Pompeu Fabra \\
        Barcelona, Spain
}

\begin{document}
\allowdisplaybreaks

\maketitle

\begin{abstract}
Qualitative numerical planning is classical planning extended with
non-negative real variables that can be increased or decreased ``qualitatively'',
i.e., by positive indeterminate amounts. While deterministic planning with numerical variables
is undecidable in general, qualitative numerical planning is decidable and provides a convenient abstract model for
generalized planning.
%Qualitative numerical planning, introduced
%by \citeay{sid:aaai2011}, showed that solutions to
%qualitative numerical problems (QNPs) correspond to the strong cyclic
%solutions of an associated fully observable non-deterministic (FOND) problem that terminate.
The solutions to qualitative numerical problems (QNPs) were shown %by \citeay{sid:aaai2011}
to correspond to the strong cyclic
solutions of an associated fully observable non-deterministic (FOND) problem that terminate.
This leads to a generate-and-test algorithm
for solving QNPs where solutions to a FOND problem are generated
one by one and tested for termination. The computational shortcomings of this
approach for solving QNPs, however, are that it is not simple to amend FOND planners to generate
all solutions, and that the number of solutions to check can be doubly exponential
in the number of variables. In this work we address these limitations while providing
additional insights on QNPs. More precisely, we introduce two polynomial-time reductions, one from QNPs to
FOND problems and the other from FOND problems to QNPs both of which do not involve termination tests.
% by reduction, it is always (implicitly) assume that is sound and complete; otherwise it is not a reduction..
A result of these reductions is that QNPs are shown to have the same expressive power
and the same complexity as FOND problems.
% The new QNP to FOND translator is implemented and available.
%, and experimental results are reported.
% Implications for FOND planning under different fairness assumptions are also discussed.
\end{abstract}

\section{Introduction}

Qualitative numerical  problems (QNPs)  are classical planning problems extended
with non-negative  numerical variables $X$ that can be  decreased or increased ``qualitatively'',
i.e., by positive indeterminate amounts. Since such numerical variables cannot  be used for counting,
QNP planning, unlike most other general forms of planning with numbers \shortcite{helmert:numeric}, turns out to be  decidable.
QNPs were introduced by \citeay{sid:aaai2011} as a useful model for \emph{generalized planning},
namely, the synthesis of plans that solve multiple classical planning instances
% \cite{levesque:loops,srivastava08learning,bonet09automatic,srivastava:generalized,hu10correctness,hu:generalized,bonet:ijcai2015,BelleL16}.%% from ijcai2017
\shortcite{levesque:loops,bonet:icaps2009,srivastava:generalized,hu:generalized,bonet:ijcai2015,BelleL16,anders:review}. %% from ijcai2017
Basically, collections $\Q$ of planning instances $P$ that share the same set of
actions and state features may be often expressed as a single QNP problem $Q$
whose solutions, that map state features into actions,  solve all problems
$P$ in $\Q$ \shortcite{bonet:ijcai2018}.
% BLAI: Following claim has been weakened since not every collection $\Q$ may be expressed as a QNP
%if $\Q$ is a collection of planning instances $P$ that share the same set of actions
%and state features, then a policy $\pi$ that maps general state features into actions and
%solves all the problems $P$ in $\Q$ can be obtained by solving a single QNP  $Q$
%that provides an   abstraction of $\Q$ \cite{bonet:ijcai2018}.

QNPs can be  solved in two steps \shortcite{sid:aaai2011}. First,  the  QNP $Q$ is converted into a standard
fully observable non-deterministic (FOND) problem $P$  \shortcite{cimatti:fond}. Then, solutions of  $P$ obtained by an off-the-shelf FOND planner
are tested for \emph{termination}. This last step is required because the non-determinism in the FOND problem $P$ is not fair
in the usual sense: the trajectories that are not fair are  those in which a variable is decreased an infinite number of times
but  increased a finite number of times only    \shortcite{bonet:ijcai2017}. The policies that solve $P$  in which the (QNP) fair trajectories
all reach the goal, and which   correspond to the  terminating strong cyclic  policies for $P$,  are the policies that solve the QNP $Q$  \shortcite{sid:aaai2011}.
% The termination test can be performed by the SIEVE algorithm introduced by \citeay{sid:aaai2011}
% and which  runs in time that is polynomial in the number of FOND states that are  reachable
% with the policy.

The computational shortcomings of solving QNPs following this  generate-and-test approach,
however, are two. First, it is not simple to amend  FOND planners to generate all the solutions
of a FOND problem because FOND plans are not action sequences but closed-loop policies.
Second, the number of policies that  need to be tested for termination may be huge:
exponential in the number of FOND states, and hence, doubly exponential in the number of variables.

In this work we address these limitations while providing
additional insights on   QNPs. We introduce  two polynomial-time
reductions, one from QNPs to FOND problems and the  other from FOND problems to
QNPs. As every (formal) reduction, both reduction are sound and complete, and hence
do not require termination tests. A result of these reductions is that QNPs
are shown to have the same expressive power and in particular, the plan-existence
decision problem for both have the same complexity EXP-Complete \shortcite{littman:fond,rintanen:po}.
The new QNP to FOND translator is implemented and available.
% and we report  experimental as well.
In combination with sound and complete FOND planners, the translation
yields   the only sound and complete QNP planner available (i.e., solver that
works directly on the factored QNP representation without explicitly flattening
the input QNP).

The structure of the paper is as follows. We review first classical planning, FOND planning, QNPs,
the direct  translation of QNPs into FOND problems, and the termination test. We follow
ideas from \citeay{sid:aaai2011} but in a slightly different and more expressive  formulation.
We then introduce the two new reductions: from  FOND problems into QNPs,  and  from QNPs into FOND problems.
This last reduction is very different from the one sketched  by \citeay{bonet:ijcai2017} which is incorrect
(we illustrate this with an example). We then consider variations and extensions of QNPs,
the use of QNPs for generalized planning,  experimental results, and  related work.

\section{Classical and FOND planning}

A classical planning problem is a  sequential decision problem
where a goal is to be reached by performing actions with
deterministic effects from a given initial state.
These problems are usually  expressed in compact
form in planning languages such as STRIPS \shortcite{fikes:strips,russell:book}.
A (grounded)  STRIPS planning problem (with negation)
is a tuple $P=\tup{F,I,O,G}$ where $F$ denotes a set of propositional variables,
$I$ and $G$ are  sets of $F$-literals representing the  initial and goal situation,
and $O$ is a set of actions $a$ with preconditions and effects
$Pre(a)$ and $\Eff(a)$ given by sets of $F$-literals.

The \emph{state model}  $\S(P)$ for the problem $P=\tup{F,I,O,G}$
is a tuple   $\S(P)=\tup{S,s_0,Act,A,f,S_G}$
where $S$ is the set of possible truth-valuations over the $F$ literals,
called the states, $s_0$ is the initial state, $Act=O$,
$A(s)$ represents the actions $a$ in $Act$ whose preconditions
are true in $s$,  $f(a,s)$ represents the state $s'$ that
follows action $a$ in $s$ for $a \in A(s)$,
and $S_G$ is the  set of goal states. It is assumed that
the problem $P$ is consistent in the sense that
$s_0$ and $f$ are well-defined and $S_G$ is not empty.
A solution to a classical  problem $P$ is an  action sequence
$a_0, \ldots, a_n$ that generates a state sequence $s_0,\ldots, s_{n+1}$
over the model $\S(P)$  that reaches the goal. In this sequence,
$a_i \in A(s_i)$ and $s_{i+1}=f(a_i,s_i)$ for $i=0, \ldots, n$,
and $s_{n+1} \in S_G$.

A \emph{fully-observable non-deterministic (FOND) problem} $P$ is like a classical planning problem
except that actions $a$ may have  non-deterministic effects
expressed as $\Eff_1(a) \, | \, \cdots \, | \Eff_n(a)$ where
$\Eff_i(a)$ is a set of $F$-literals as above \shortcite{cimatti:fond,geffner:book,ghallab:new-book}.
The state model  $\S(P)$ determined by a FOND problem $P=\tup{F,I,O,G}$
is a tuple   $\S(P)=\tup{S,s_0,Act,A,F,S_G}$ as above with the difference
that the state transition function $F$ is non-deterministic, and
maps an action $a$ and state $s$ into a non-empty set $F(a,s)$ of possible successor
states.
As usual, the non-deterministic transition function $F$ is given
in factored form. That is, for action $a$ made of multiple effects
$\Eff_1\,|\,\cdots\,|\,\Eff_n$ (possibly deterministic when $n=1$),
each outcome $s'$ in $F(a,s)$ results of the choice of one $\Eff_i$
for each non-deterministic effect of $a$.\footnote{\label{foot:1}As it is
  standard, any choice of effects is assumed to be consistent (i.e., any pair of
  choices for two different non-deterministic effects of the
  \emph{same action} contain no complementary literals).
  However, with some (polynomially bounded) extra work, our methods, algorithms
  and results still apply if the model is extended with \emph{constraints} that
  every outcome $s'$ must satisfy, when such constraints are given in
  suitable form; e.g.\ DNF formulas over $F$.
}

The solutions of FOND problems ensure that the goal is reached  with certainty
under certain  fairness assumptions. Policies or plans in the FOND setting are
partial functions $\pi$ mapping states $s$ into actions $\pi(s)$. A state
trajectory $s_0, s_1, \ldots, s_n$ (finite or infinite) is induced by $\pi$
over the model $\S(P)$ if the action $a_i=\pi(s_i)$ is defined,
it is applicable in the state $s_i$, i.e., $a_i \in A(s_i)$,
and $s_{i+1}$ is in $F(a_i,s_i)$, $i=1, \ldots, n-1$.
The trajectory is  said to be a \emph{$\pi$-trajectory.}
The trajectory is \emph{maximal} if A)~it is  infinite, i.e.,  $n=\infty$,
and does not include a goal state, B)~if $s_n$ is the first  goal state in the sequence, or
C)~the action $\pi(s_n)$ is not defined or not applicable in $s_n$.

A  policy $\pi$ is a solution of the FOND problem $P$ if all the \emph{fair}
maximal trajectories induced by $\pi$ over the model $\S(P)$ are goal reaching
\shortcite{strong-cyclic,cimatti:fond}. The so-called \emph{strong solutions} assume that all state trajectories are
fair. \emph{Strong-cyclic solutions}, on the other hand, assume  that all trajectories are fair \emph{except}
the infinite trajectories where a state $s$ occurs infinitely often but a  state transition $(s,s')$
for some $s' \in F(a,s)$ for $a=\pi(s)$, occurs finitely often.
The latter trajectories are deemed to be  \emph{unfair}.

Other \emph{equivalent} characterizations of strong and strong cyclic solutions are common.
For example, a strong cyclic solution $\pi$ for a FOND problem $P$ is also a policy
$\pi$ such that for each $\pi$-trajectory connecting an initial state to a state $s$,
there is a $\pi$-trajectory connecting $s$ to a goal state. Similarly, a strong solution is
a strong cyclic solution $\pi$  with no cycles; i.e., one where   no $\pi$-trajectory visits
the same state twice.

Strong solutions  can also  be thought as winning strategies against an adversary,
while strong cyclic solutions as winning strategies against nature. Indeed,
there is a well known relation between (proper)  policies that achieve the goal with probability
1 in goal-based MDPs (Markov Decision Processes) and the strong cyclic policies
that solve the FOND problem associated with the MDP, where the transition function
is such that $F(a,s)$ collects the states $s'$ that are possible
after action $a$ in  $s$, i.e., for which $P_a(s'|s) > 0$ \shortcite{geffner:book}.

From now, by solution of a FOND problem we mean a \emph{strong cyclic solution} or \emph{plan},
and by a FOND planner, we mean a strong cyclic planner, i.e., a planner that produces strong cyclic solutions.
There are some good FOND planners available, including PRP \shortcite{prp}, based on classical planners,
MyND \shortcite{mynd}, based on heuristic AND/OR search, and FOND-SAT \shortcite{geffner:fond-sat}, based on a reduction to SAT.

\section{Qualitative Numerical Problems}

\emph{Qualitative numerical problems  (QNPs)} are  classical
planning problems extended with numerical variables that can be
decremented or incremented ``qualitatively''. We make this formal below.

\subsection{QNPs: Syntax}

The syntax of QNPs is defined as an extension of the  STRIPS language
with negation. A QNP is a tuple $Q=\tup{F,V,I,O,G}$ where the new
component  is a set $V$ of \emph{non-negative numerical variables} $X \in V$.
These variables introduce the  non-propositional atoms $X=0$ and their negations,
denoted as $X > 0$. These  literals can appear in the initial situation, action preconditions, and goals
of $Q$. The effects of actions $a$ on a numerical variable $X$ can be  only qualitative increments
or qualitative decrements  denoted by the expressions $Inc(X)$ and $Dec(X)$, often abbreviated
as $\INC{X}$ and $\DEC{X}$ respectively.  We refer to $X=0$ and $X > 0$ as the $V$-literals for $X \in V$, and
to $p$ and $\neg p$ for $p \in F$, as the $F$-literals in $Q$.

\begin{definition}
  \label{def:qnp:syntax}
  A QNP is a tuple $Q=\tup{F,V,I,O,G}$ where $F$ and $V$ are  sets of propositional and numerical
  variables respectively, $I$ and $G$ denote the initial and goal situations, and $O$ is a set of
  actions $a$ with preconditions,  and propositional  and numerical effects that are denoted
  as $Pre(a)$, $\Eff(a)$, and $N(a)$ respectively. The $F$-literals can appear in
  $I$, $G$, $Pre(a)$, and $\Eff(a)$, while $V$-literals can appear in $I$, $G$, and $Pre(a)$.
  The numerical effects $N(a)$ only contain special atoms of the form $Inc(X)$ or $Dec(X)$ for
  the variables $X$ in $V$. Actions with the $Dec(X)$ effect must feature the precondition
  $\GT{X}$ for any variable $X$ in $V$.
\end{definition}

% A standard closed world assumption will be made regarding the initial situation: the literals  $\neg p$ and $X > 0$ are assumed to be in $I$
% when $I$ does not contain a  $p$ or $X$ literal $p \in F$ and $X \in V$.

The preconditions and  effects  of  an action $a$ are denoted as  pairs $\abst{Pre(a)}{\Eff(a), N(a)}$
where $N(a)$ contains the numerical effects; namely, expressions like  $\INC{X}$ and $\DEC{X}$ that
stand for the increment and decrement of variable $X$ respectively.
QNPs are assumed to be syntactically consistent by requiring that no pair
of complementary literals or qualitative effects appears in the initial situation,
action effects, or goals. A pair of complementary literals or qualitative effects
has the form $\{p,\neg p\}$ for some $p$ in $F$, or $\{\EQ{X},\GT{X}\}$ or $\{\DEC{X},\INC{X}\}$
for some $X$ in $V$.
%Before defining  the semantics of QNPs, let us consider an example.

\begin{example}
An  \emph{abstraction}  that is suitable for expressing the generalized problem
of achieving the goal $clear(x)$ over an  arbitrary Blocksworld instance \shortcite{bonet:ijcai2018}
is given  in terms of the QNP $Q_{clear}=\tup{F,V,I,O,G}$ where   $F=\{H\}$
contains a boolean variable $H$  that represents if the  gripper is holding a block, $V=\{n(x)\}$  contains a numerical
variable  $n(x)$ that represents  the number of blocks above $x$, and  $I=\{\neg H, \GT{n(x)}\}$ and  $G=\{\EQ{n(x)}\}$
represent the initial and goal situations. The actions $O=\{a,b\}$ are
\begin{alignat}{1}
  \label{eq:ex1a}
  a\ &=\ \abst{\neg H, \GT{n(x)}}{H, \DEC{n(x)}}
  \intertext{and}
  \label{eq:ex1b}
  b\ &=\ \abst{H}{\neg H} \,.
\end{alignat}
\indent It is easy to see that the first action $a$ picks up blocks that are above $x$;
its first precondition $\neg H$ expresses that the gripper is holding no block,
while the second $\GT{n(x)}$ that there is at least one block above $x$.
The effects, on the other hand, make $H$ true (expressing that some block is being held)
and decrease the number $n(x)$ of blocks above $x$.
The other action $b$ puts the  block  being held away from block $x$ as expressed
by the  precondition $H$ and effect $\neg H$. The fact that $b$ puts blocks
away from block $x$ is reflected in that it does not affect the variable $n(x)$.

The QNP $Q_{clear}$  captures the relevant part of the infinite collection of Blocksworld instances
where the goal is to achieve the atom $clear(x)$ for some block $x$. The solution to $Q_{clear}$
provides the general strategy for solving all such instances.
For ways of learning such abstractions automatically; see the recent work of \citeay{bonet:aaai2019}.
\end{example}

\subsection{QNPs: Semantics}

A state $s$ for QNP $Q=\tup{F,V,I,O,G}$ is a valuation that assigns a truth
value $s[p]$  to each boolean variable $p \in F$, and a non-negative real value $s[X]$
to each numerical variable $X \in V$. Since the initial situation $I$ can only feature
atoms of the form $\EQ{X}$ or $\GT{X}$, there is a \emph{set} $S_0$ of possible initial states $s_0$
that correspond to the valuations that satisfy the literals in $I$.
For example, in $Q_{clear}$, $I$ is given by the literals $I=\{\neg H, \GT{n(x)}\}$,
meaning that $S_0$ contains all and only the valuations that make $H$ false and $n(x){\,=\,}r$
for some positive real number $r$. The use of variables that can take real values
for representing integer counters illustrates that the semantics of QNPs is coarse-grained,
and  for  this reason, decidable. Indeed, QNPs use just one qualitative property of numbers; namely,
that a non-negative  variable eventually must reach the value of zero if it keeps being decremented
and not incremented. This property is true for integers, and it also true for  reals,
as long as the magnitude of the decrements is bounded from below by some positive $\epsilon$-parameter.
More about this below.
The state model  $\S(Q)$ represented  by a QNP can be characterized as follows:

\begin{definition}
  \label{def:qnp:state-model}
  A QNP  $Q=\tup{F,V,I,O,G}$ determines a non-deterministic state model $\S(Q)=\tup{S,S_0,Act,A,F,S_G}$ where
  \begin{enumerate}[$\bullet$]
    \item the states $s$ in $S$ are the valuations that assign  a truth  value to
      the boolean variables in $F$ and a non-negative real value to the numerical variables in $V$,
    \item the initial states $s_0$ in $S_0$ are those that satisfy the literals in $I$ under a closed-world
      assumption ($s_0$ makes $p$ and $\EQ{X}$ false if the literals $p$ and $\EQ{X}$ are not in $I$),
    \item the actions in $Act$ are those in $O$; i.e.,  $Act=O$,
    \item the actions $A(s)$ applicable in $s$ are those in $Act$ such that $Pre(a)$ is true in $s$,
    \item the goal states in $S_G$ are those that satisfy $G$,
    \item the transition function $F$ is such that $s' \in F(a,s)$ for $a \in A(s)$ if
      \begin{enumerate}
        \item $s'[p]$ is $true$ (resp. $false$) if $p$ (resp. $\neg p)$ is in $\Eff(a)$,
        \item $s[X] < s'[X]$ if $Inc(X)$ is in $N(a)$,
        \item $s'[X] < s[X]$ if $Dec(X)$  is in $N(a)$,
        \item $s'[p] = s[p]$ if neither $p$ nor $\neg p$ in $\Eff(a)$,
        \item $s'[X] = s[X]$ if neither $\INC{X}$ nor $\DEC{X}$ in $N(a)$.
      \end{enumerate}
  \end{enumerate}
\end{definition}

A \emph{trajectory} $s_0, a_0, s_1, a_1, \ldots, s_n$
is compatible with the model $\S(Q)=\tup{S,S_0,Act,A,F,S_G}$ if
$s_0 \in S_0$, and $a_i \in A(s_i)$ and $s_{i+1} \in F(a_i,s_i)$ for each $a_i$ in the sequence.
The   trajectory is an $\epsilon$-bounded trajectory  or  \emph{$\epsilon$-trajectory}
if the  numerical changes are bounded from below by a parameter $\epsilon > 0$,
except when this would make the variable negative:

\begin{definition}
  \label{def:qnp:epsilon-trajectory}
  A trajectory   $s_0, a_0, s_1, a_1, \ldots, s_n$  is an \emph{$\epsilon$-trajectory}
  iff for any variable $X$ and time point $i$, with $i < n$,
  $s_{i+1}[X] \neq s_{i}[X]$ implies $|s_{i+1}[X] -  s_{i}[X]| \geq \epsilon$ or $s_{i+1}[X]=0$.
  %except when $s_{i}[X] < \epsilon$  where $s_{i+1}[X] < s_i[X]$ implies  $s_{i+1}[X]=0$.
\end{definition}

Trajectories bounded by  $\epsilon > 0$  cannot decrease the value of a variable asymptotically
without ever reaching the value of zero. This is  in agreement with the key assumption in QNPs
by which  variables  that keep being decreased and not  increased  eventually must reach the
value  zero. From now, \textbf{trajectories over QNPs will refer to $\epsilon$-trajectories for some $\epsilon > 0$.}

\subsection{QNPs: Solutions}

Solutions to QNPs  take the form of partial functions or policies $\pi$ that  map  states into actions.
The choice of the action $\pi(s)$ to be done in a state $s$,  however, can only depend
on the truth  values $s[p]$  associated with the boolean variables $p$ in $F$
and the truth values of the expressions $s[X]=0$ associated with the numerical variables
$X$ in $V$. If we use the notation  $s[X=0]$ to refer to $s[X]=0$,
then $\pi(s)$ must depend solely  on the \emph{truth-valuation over the $F$-literals $p$ and the
  $V$-literals $X=0$} that are determined by the state $s$. There is indeed a finite number of  such truth valuations
but an  infinite number of states. We refer to such truth  valuations as the  \emph{boolean states}
of the QNP and denote the boolean state associated with a state $s$ as $\bar s$.

\begin{definition}[Policy]
  \label{def:qnp:policy}
  A policy $\pi$ for a QNP $Q=\tup{F,V,I,O,G}$  is a partial mapping of states  into actions
  such that $\pi(s)=\pi(s')$ if $\bar s=\bar s'$.
\end{definition}

A  trajectory $s_0, a_0, s_1, a_1, \ldots,s_n$ compatible with the model
$\S(Q)$ is said to be a  \emph{$\pi$-trajectory} for $Q$ if  $a_i = \pi(s_i)$.
Sometimes, a $\pi$-trajectory is simply denoted as a sequence of states since the actions are determined by $\pi$.
A $\pi$-trajectory is also said to be a \emph{trajectory induced by} $\pi$ or compatible with $\pi$.
As before, a $\pi$-trajectory is \emph{maximal} if A) the trajectory is infinite  and does not include a goal state,
B)~$s_n$ is the first goal state in the trajectory, or C)~$\pi(s_n)$ is undefined or denotes  an action that is not applicable in $s_n$.
The solutions to QNPs are defined then  as follows:

\begin{definition}[Solution]
  \label{def:qnp:solution}
  Let $Q$ be a QNP and let $\pi$ be a policy for $Q$.
  The policy  $\pi$ \emph{solves}  $Q$ iff for every $\epsilon>0$, all the maximal
  $\epsilon$-trajectories induced by $\pi$ reach a goal state.
\end{definition}

We will see that solutions to QNPs can be characterized equivalently  in terms of a suitable notion of \emph{QNP-fairness};
namely,  a policy $\pi$ solves $Q$ iff  every  maximal (QNP) fair trajectory induced by $\pi$ reaches the goal,
where the \emph{unfair} trajectories are those in which some variable $X$ is decreased infinitely often but increased
finitely often. For historical reasons,  such unfair trajectories are  called \emph{terminating} instead,
as indeed they cannot go on  forever if the decrements are bounded from below by some $\epsilon > 0$.

\begin{example}
Consider the QNP $Q_{clear}=\tup{F,V,I,O,G}$ from above with $F=\{H\}$, $V=\{n(x)\}$,
$I=\{\neg H,\GT{n(x)}\}$, $G=\{\EQ{n(x)}\}$, and $O=\{a,b\}$ where
$a=\abst{\neg H, \GT{n(x)}}{H, \DEC{n(x)}}$ and $b=\abst{H}{\neg H}$.
Let $\pi$ be the policy defined by the rules:
\begin{alignat}{1}
  &\text{if $\neg H$ and $\GT{n(x)}$, then do $a$} \,, \\
  &\text{if $H$ and $\GT{n(x)}$, then do $b$} \,.
\end{alignat}
All the maximal $\epsilon$-bounded trajectories that are induced by
the policy $\pi$ on $Q_{clear}$ have the form
\begin{alignat}{1}
  s_0,a,s_1,b,s_2,a,s_3,b,\ldots,s_{2m},a,s_{2m+1}
\end{alignat}
where $s_{m+1}$, for positive integer $m$, is the first state where $\EQ{n(x)}$ is true.
The actions $a$ and $b$ alternate because the first makes $H$ false and the second
makes it true.
In each transition $(s_i,s_{i+1})$ for a non-negative even integer $i$, the numerical
variable $n(x)$ decreases by $\epsilon$ or more,  unless $s_{i+1}[n(x)]=0$.
%must be such that $0 < s_i[n(x)] < \epsilon$.  In the last case, the state $s_{i+1}[n(x)]$ must be $0$
The former case cannot happen more than $s_0[n(x)]/2\epsilon$ times, as the numerical
variable $n(x)$ is decreased every two steps and is never increased.
Thus, in all cases and for any $\epsilon > 0$, any $\epsilon$-trajectory induced
by the policy $\pi$ reaches a goal state in a finite number of steps, regardless
of the initial value $s_0[n(x)]$ of $n(x)$, and regardless of the actual magnitude
of the changes $|s_{i+1}[n(x)]-s_{i}[n(x)]|$.
\end{example}

\medskip

\begin{example}
A more interesting QNP that requires ``nested loops'' is  $Q_{nest}=\tup{F,V,I,O,G}$ with
$F=\emptyset$, $V=\{X,Y\}$, $I=\{\GT{X},\GT{Y}\}$, $G=\{\EQ{X}\}$, and $O=\{a,b\}$ where
\begin{alignat}{1}
  a\ &=\ \abst{\GT{X}, \EQ{Y}}{\DEC{X}, \INC{Y}} \,, \\
  b\ &=\ \abst{\GT{Y}}{\DEC{Y}} \,.
\end{alignat}
The policy $\pi$ is given by the rules:
\begin{alignat}{1}
  &\text{if $\GT{X}$ and $\EQ{Y}$, then do $a$} \,, \\
  &\text{if $\GT{X}$ and $\GT{Y}$, then do $b$} \,.
\end{alignat}
The policy decrements $Y$ using action $b$ until the action $a$ that decreases $X$
and  increases $Y$ can be applied, and  the process is  repeated until $\EQ{X}$.
The $\epsilon$-trajectories  induced by $\pi$ have  the form
\begin{alignat}{1}
  s_0, b, \ldots, b, s^1_{k_0},  %%% added  b loop before ..
      \quad  a, s^1_0, b, \ldots, b, s^1_{k_1},
      \quad  a, s^2_0, b, \ldots, b, s^2_{k_2},
      \quad  \ldots,
      \quad  a, s^m_0, b, \ldots, b, s^m_{k_m},
      \quad  a,s_G \,.
\end{alignat}
where there is an outer loop that  is  executed a number of times $m$
bounded by $s_0[X]/\epsilon$,  as  $X$ is decreased by $\epsilon$ or more,
but is not   increased. In the iteration $i$ of such a loop,
the action $b$ is executed a number  of times  $k_i$ bounded by $s^i_0[Y]/\epsilon$
as in such inner  loop $Y$ begins with value $s^i_0[Y]$ and it is decreased and not increased.
The result is that  all the $\epsilon$-trajectories induced by $\pi$ reach
a goal state in a finite number of steps that cannot be bounded a priori because the
increments of $Y$ produced by  the action $a$ are finite but  not bounded.
The policy $\pi$  thus solves $Q_{nest}$.
\end{example}

\section{Direct Translation and Termination Test}

The problem of deciding the existence of a policy that solves a given QNPs is decidable
as noted by \citeay{sid:aaai2011}. They hint a generate-and-test procedure to find such
a policy where the QNP is first translated into a FOND problem, and then all the possible
strong cyclic policies for the FOND problem are enumerated and tested for termination.
The translation runs in polynomial (linear) time in the number of boolean states for the
QNP while the termination test for a given strong cyclic solution is polynomial in the number
of FOND states.
However, the number of strong cyclic solutions that need to be tested is exponential
in the number of FOND states in the worst case.
The generate-and-test approach is not efficient but it is complete and runs in finite time.
In contrast, the plan-existence problem for \emph{numerical planning is undecidable} even
in the classical setting where there is a single initial state and the action effects are
deterministic; e.g., Post's correspondence problem can be reduced to a numerical planning
problem \shortcite{helmert:numeric}.
The decidability of plan existence for QNPs is due to the ``qualitative'' behaviour of
the numerical variables that cannot keep track of counts; in particular, the variables
cannot be incremented or decremented by specific amounts nor queried about specific values.
We review the translation and the termination test for QNPs before considering a novel polynomial
translation which does not require termination tests and which thus is a true reduction of QNPs
into FOND problems.

The translation $T_D$ from  a QNP $Q$ to a FOND $P=T_D(Q)$ by \citeay{sid:aaai2011} is simple and direct,
and it involves three steps: 1)~the literals $\EQ{X}$ and $\GT{X}$ are made propositional
with the numerical variables $X$ eliminated, 2)~$Inc(X)$ effects are converted into deterministic
boolean effects $\GT{X}$, and 3)~$Dec(X)$ effects are converted into \emph{non-deterministic}
boolean effects $\GT{X}\,|\,\EQ{X}$.

\begin{definition}[Direct Translation $T_D$]
  \label{def:td}
  For QNP $Q=\tup{F,V,I,O,G}$, the FOND problem $P=T_D(Q)$ is $P=\tup{F',I',O',G'}$ with
  \begin{enumerate}[1.]
    \item $F'=F \cup \{\EQ{X} \,:\, X\in V\}$, where $\EQ{X}$ stands for a new propositional
      symbol $p_{\EQ{X}}$ and $\GT{X}$ stands for $\neg p_{\EQ{X}}$,
    \item $I'=I$ but with $\EQ{X}$ and $\GT{X}$ denoting  $p_{\EQ{X}}$ and  $\neg p_{\EQ{X}}$,
    \item $O'=O$  but with $Inc(X)$ effects replaced by the deterministic propositional
      effects $\GT{X}$, and $Dec(X)$ effects replaced by non-deterministic propositional
      effects $\GT{X}\,|\,\EQ{X}$,
    \item $G'=G$ but with $\EQ{X}$ and $\GT{X}$ denoting $p_{\EQ{X}}$ and  $\neg p_{\EQ{X}}$.
  \end{enumerate}
\end{definition}

The problem $P=T_D(Q)$ is a special type of FOND problem. For example, from its definition,
there is no action in  $P$ that can achieve a  proposition $\EQ{X}$ deterministically.
We refer to actions in the FOND $P$ with effects $\GT{X}$ and $\GT{X}\,|\,\EQ{X}$ as
$Inc(X)$ and $Dec(X)$ actions, as such effects in $P$ may only come from $Inc(X)$ and
$Dec(X)$ effects in $Q$. Also, observe that the FOND $P$ has a unique initial state
even though the QNP $Q$ may have an infinite number of initial states.

The states of the FOND problem $P=T_D(Q)$ are related to the \emph{boolean states}
over  $Q$, i.e., the truth-assignments over the atoms $p$ and  $\EQ{X}$, the latter of which
stand for  (abbreviation of) symbols in $P$.
A policy $\pi$ for the QNP $Q$ thus induces a policy over the FOND problem $P$ and vice versa.\footnote{The policy $\pi$ over states $s$
  of $Q$ determines the policy $\pi'$ over the FOND $P$ where $\pi'(t)=\pi(s)$ if $t=\bar s$,
  and vice versa, a policy $\pi'$ for $P$ determines a  policy $\pi$ for $Q$
  where $\pi(s)=\pi'(t)$ if $\bar s=t$. For simplicity, we use the same notation $\pi$
  to refer to the policy $\pi$ over $Q$ and the policy $\pi'$ that it induces over $P=T_D(Q)$.}
Moreover, the  FOND problem $P=T_D(Q)$ captures the \emph{possible boolean state transitions} in $Q$
exactly. More precisely,  $(s,a,s')$ is a possible transition in $Q$ iff $(\bar s,a,\bar s')$ is
a possible transition in $P$.  Indeed, if  we extend  the notion of strong cyclic policies to QNPs:

\begin{definition}
  \label{def:qnp:strong-cyclic}
  Let $Q$ be a QNP and let $\pi$ be a policy for $Q$.
  $\pi$ is strong cyclic for $Q$ iff for every $\pi$-trajectory connecting $s_0$ with
  a state $s$, there is a $\pi$-trajectory connecting $s$ with a goal state.
\end{definition}

\noindent
% Recall that trajectories in QNPs refer to $\epsilon$-trajectories for some $\epsilon > 0$.
The following correspondence between boolean states in $Q$ and the states of the boolean FOND problem $T_D(Q)$
results:

\begin{theorem}
  \label{thm:strong-cyclic}
  Let $Q$ be a QNP and let $\pi$ be a policy for $Q$.
  $\pi$ is strong cyclic solution for $Q$ iff $\pi$ is strong cyclic policy for the FOND problem $T_D(Q)$.
\end{theorem}
\begin{proof}
  Let $\S(Q)=\tup{S,S_0,Act,A,F,S_G}$ and $\S(P)=\tup{S',s'_0,Act',A',F',S'_{G}}$ be the state models
  for the QNP $Q$ and the FOND problem $P=T_D(Q)$. From the definition of the $T_D$ translation,
  the state $s$ is in $S_0$ (resp.\ in $S_G$) iff $\bar s=s'_0$ (resp.\ in $S'_{G}$),
  and  the state $s' \in F(a,s)$ for $a \in A(s)$ iff $\bar s' \in F'(a,\bar s)$ for $a \in A'(\bar s)$.
  This means that there is a  $\pi$-trajectory  connecting an initial state $s_0$ in $S_0$ with a  state $s$ in $S$
  iff there is a corresponding $\pi$-trajectory connecting $s'_0$  with $\bar s$ in $S'$, and similarly,
  there is a $\pi$-trajectory connecting $s$ with a goal state $s'$ iff there is a corresponding $\pi$-trajectory
  connecting $\bar s$ with $\bar s'$ in $\S(Q)$.
\end{proof}

The correspondence between the $\pi$-trajectories connecting states $s$ in $Q$ and the  $\pi$-trajectories connecting the states
$\bar s$ in  $P=T_D(Q)$  does not imply however  that the solutions of $P$ and $Q$ are the same.
Indeed, the $Dec(x)$ effects of an action $a$  in $Q$  are mapped into the  non-deterministic propositional
effects $\GT{X}\,|\,\EQ{X}$ in $P=T_D(Q)$ which  implies that $\EQ{X}$ will be true in $P$ if the action $a$ is repeated
infinitely often. On the other hand, a $Dec(X)$ effect in $Q$ ensures that $\EQ{X}$ will be true if $a$ is repeated
infinitely often \emph{as long as no $Inc(X)$ action is performed infinitely often as well}.

In other words, the correspondence between the state transitions $(s,a,s')$ in $Q$ and the state transitions $(\bar s,a,\bar s')$ in
$P=T_D(Q)$ does  not extend to \emph{infinite trajectories} \shortcite{bonet:ijcai2017}. Recall that trajectories in $Q$
refer to $\epsilon$-trajectories for some $\epsilon > 0$ that exclude ``infinitesimal'' changes. As a result:

\begin{theorem}
  \label{thm:td:gap}
  Let $Q$ be a QNP and let $\pi$ be a policy for $Q$.
  If $\tau = s_0, s_1 \ldots$ is an infinite $\pi$-trajectory in $Q$, then
  $\bar\tau = \bar s_0, \bar s_1, \ldots$ is an infinite $\pi$-trajectory in $P=T_D(Q)$.
  %Yet if $\bar\tau = \bar s_0, \bar s_1, \ldots$ is an infinite $\pi$-trajectory in $P=T_D(Q)$,
  %there may not be an infinite trajectory $\tau = s_0, s_1 \ldots$ in $Q$.
  Yet, there may be infinite $\pi$-trajectories in $P=T(D)$ that do not correspond to any $\pi$-trajectory in $Q$.
\end{theorem}
\begin{proof}
  For the first part, if $\tau$ is an infinite $\pi$-trajectory over $Q$, then  $s_{i+1} \in F(a_i,s_i)$ for $a_i=\pi(s_i)$;
  therefore  $\bar s_{i+1} \in F(a_i,\bar s_i)$ for $a_i = \pi(\bar s_i)$, and hence
  $\bar\tau = \bar s_0, \bar s_1, \ldots$ is an infinite $\pi$-trajectory over $P$.

  For the second part, one example suffices. Let $Q$ be a QNP with a single  variable $X$ that is numerical,
  a single action $a$ with precondition $\GT{X}$ and effect $Dec(X)$, initial condition $\GT{X}$, and goal $\EQ{X}$.
  In the state model $\S(P)$ associated with the FOND problem $P=T_D(Q)$, there are two states $t$ and $t'$,
  the first where $\EQ{X}$ is true and the second where $\GT{X}$ is true, and there  is an  infinite trajectory
  $\bar s_0, \bar s_1, \ldots$ where all $\bar s_i=t'$ and $\pi(\bar s_i)=a$,
  but there is no infinite trajectory $s_0, s_1 \ldots$ in $Q$ where $\GT{X}$ stays true
  forever while being decremented. Indeed, for any $\epsilon > 0$ and any initial value of $X$, $s_0[X]>0$,
  it is the case that $s_n[X]=0$ for $n > s_0[X]/\epsilon$.
\end{proof}

The notion of \emph{termination} is aimed at capturing the infinite $\pi$-trajectories over the FOND problem
$P=T_D(Q)$ that do not map into infinite $\pi$-trajectories over $Q$.
Let
\begin{alignat}{1}
  \bar s_0, \bar s_1, \ldots, [\bar s_i, \ldots, \bar s_{m}]^*
\end{alignat}
denote \emph{any infinite} $\pi$-trajectory on the FOND $P$ where the states $\bar s_i, \ldots, \bar s_{m}$
in brackets make the non-empty set of \emph{recurring states}; namely those that occur infinitely
often in the trajectory (not necessarily in that order). We refer to such set of recurrent states
as the \emph{loop} of the trajectory. Observe that knowledge of the loop (and the policy $\pi$) is
sufficient to infer whether a variable $X$ is decremented or incremented infinitely often.
Termination imposes the following condition on loops:

\begin{definition}[Terminating Trajectories]
  \label{def:td:terminatig-trajectories}
  Let $Q$ be a QNP and let $\pi$ be a policy for $Q$.
  An infinite $\pi$-trajectory $\bar s_0, \ldots, [\bar s_i, \ldots, \bar s_{m}]^*$ is \emph{terminating} in
  $P=T_D(Q)$ if there is a variable $X$ in $Q$ that is decremented but not incremented in the
  loop; i.e., if $\pi(\bar s_k)$ is a $Dec(X)$ action for some $k \in [i,m]$,  and  $\pi(\bar s_j)$ is
  not an  $Inc(X)$ action for any $k \in [i,m]$.
\end{definition}

The notion of termination is a notion of fairness that is different from the one underlying
strong cyclic planning that says that infinite but terminating trajectories in $P$ are not
``fair'' and hence can be ignored.
Indeed, this  notion of termination closes the gap in Theorem~\ref{thm:td:gap}:

\begin{theorem}
  \label{thm:termination1}
  Let $Q$ be a QNP and let $\pi$ be a policy for $Q$.
  $\bar\tau = \bar s_0, \bar s_1, \ldots$ is an infinite \emph{non-terminating} $\pi$-trajectory in $P=T_D(Q)$
  iff there is  an infinite $\pi$-trajectory $\tau = s_0, s_1, \ldots$ in $Q$.
\end{theorem}
\begin{proof}
  Let $\tau = s_0, s_1, \ldots$ be an infinite $\pi$-trajectory in $Q$, and let us assume
  that the infinite trajectory $\bar\tau = \bar s_0, \bar s_1, \ldots$ is terminating.
  Then there must be a variable $X$ that is decremented by $\pi(s)$ in some recurring state
  $s$ in $\tau$ and which is not incremented by $\pi(s')$ on any recurrent state $s'$ in $\tau$.
  Let $s(t)$ denote the state at time point $t$ in $\tau$, let $t$ be the last time point where
  variable $X$ is increased in $\tau$ ($t=-1$ if $X$ is not increased in $\tau$), and let
  $X(t+1)$ be the value of variable $X$ at the next time point.
  The maximum number of times that $X$ can be decreased after $t+1$ is bounded by $X(t+1)/\epsilon$,
  and after this, $X$ must have zero value. But in $\tau$, $X$ is decreased an infinite number
  of times, in contradiction with the assumption that any action that decrements $X$ features $\GT{X}$
  as precondition. %$\tau$ is an  infinite trajectory.

  For the converse, we show that one such trajectory $\tau$ in $Q$ can be constructed for any
  $\epsilon >0$, given that the trajectory $\bar\tau$ in $P$ is non-terminating. We do so
  by adjusting the non-deterministic increments and decrements of the actions, all of which
  have to be greater than or equal to $\epsilon$, except when this would result in negative
  values that are increased back to zero.
  We construct $\tau = s_0, s_1, \ldots$ from $\bar\tau = \bar s_0, \bar s_1, \ldots$ as follows.
  The value of the boolean variables is the same in $s_i$ as in $\bar s_i$, and in addition,
  $s_i[X]=0$ iff $\EQ{X}$ is true in $\bar s_i$ for all $i$.
  We just have to find exact values for the numerical variables $X$ in each of the states $s_i$
  in $\tau$, and this is a function of  their initial values $s_0[X]$ when $s_0[X]>0$, and the
  positive decrements or increments $\Delta(X,s_i)$ when $\pi(s_i)$ is a $Dec(X)$ or $Inc(X)$ action,
  and $\Delta(X,s_i) \ge \epsilon$. For simplicity and without loss of generality, let us assume
  that $\epsilon < 1$ (the case for $\epsilon\geq 1$ is an easy exercise).

  All the positive initial values of numerical variables, increments, and decrements are set
  to \emph{positive integers} by considering the sequence of actions $\pi(s_i)$, $i=1,\ldots$.
  The initial values $s_0[X]$ are set to $1 + k(X,0)$ where $k(X,i)$ stands for the number
  of $Dec(X)$ actions that occur between the state $s_i$ and the first state $s_j$ after $s_i$
  where an $Inc(X)$ action occurs (if no $Inc(X)$ action occurs after state $s_i$, $k(X,i)$ is
  the number of $Dec(X)$ actions after $s_i$). That is, $k(X,i)$ is the cardinality of the set
  \begin{alignat}{1}
    \{ j : \text{$i \leq j < ind(X,i)$ and $\pi(s_j)$ is $Dec(X)$ action} \}
  \end{alignat}
  where $ind(X,i)$ is the minimum index $j > i$ such that $\pi(s_j)$ is an $Inc(X)$ action,
  or $\infty$ if there is no such action after $s_i$. Observe that $k(X,i)$ is bounded.
  The only way it could be infinite is when no $Inc(X)$ action occurs after $s_i$ while at
  the same time an infinite number of $Dec(X)$ actions occur; yet, this is impossible since
  then $X$ eventually becomes zero after which no $Dec(X)$ action may occurs as such actions
  feature the precondition $\GT{X}$.
  Likewise, the increments $\Delta(X,s_i)$ are set to $1 + k(X,i)$, and the decrements
  $\Delta(X,s_i)$ are set to $s_i[X]$ if $\EQ{X}$ is true in $\bar s_{i+1}$  and to $1$
  is $\GT{X}$ if true in $\bar s_{i+1}$.
  It is not difficult to verify that these choices define a trajectory $\tau$ in $Q$ that
  corresponds to the assumed trajectory $\bar\tau$ in $P$.
\end{proof}

The full correspondence between infinite $\pi$-trajectories $\tau$ in $Q$ and infinite
non-terminating $\pi$-trajectories $\tau$ in $P=T_D(Q)$ suggests the following
definition of \emph{termination} in QNPs $Q$ and FOND problems $T_D(Q)$:

\begin{definition}[Termination in $Q$]
  \label{def:qnp:termination}
  A policy $\pi$ for the QNP $Q$ is terminating iff all the $\pi$-trajectories on $Q$ are of finite length.
  In such a case, we say that $\pi$ is $Q$-terminating.
\end{definition}

\begin{definition}[Termination in $P$]
  \label{def:td:termination}
  %Let $Q$ be a QNP.
  A policy $\pi$ for the FOND problem $P=T_D(Q)$ is terminating iff all the infinite $\pi$-trajectories
  on $P$ are terminating.
  In such a case, we say that $\pi$ is $P$-terminating.
\end{definition}

\noindent The correspondence between policies can then be expressed as:

\begin{theorem}
  \label{thm:termination2}
  Let $Q$ be a QNP, let $P=T_D(Q)$ be its direct translation, and let $\pi$ be a policy for $Q$ (and thus also for $P$).
  Then, $\pi$ is $Q$-terminating iff $\pi$ is $P$-terminating.
\end{theorem}
\begin{proof}
  Direct from Theorem~\ref{thm:termination1}.
  For one direction, assume that $\pi$ is $Q$-terminating and let $\bar\tau$ be a $\pi$-trajectory in $P$.
  If $\bar\tau$ is not terminating, by Theorem~\ref{thm:termination1}, $\tau$ is infinite and thus $\pi$
  would not be $Q$-terminating. Therefore, every $\pi$-trajectory $\bar\tau$ in $P$ is terminating and
  thus $\pi$ is $P$-terminating. The other direction is established similarly.
\end{proof}

\noindent The \emph{soundness and completeness}  of the direct translation extended with
termination can be expressed as following:

\begin{theorem}[Soundness and Completeness $T_D$]
  \label{thm:td:main}
  Let $Q$ be a QNP, let $P=T_D(Q)$ be its direct translation, and let $\pi$ be a policy for $Q$ (and thus also for $P$).
  The following are equivalent:
  \begin{enumerate}[1.]
    \item $\pi$ solves $Q$,
    \item $\pi$ is a strong cyclic solution of $Q$ and $\pi$ is $Q$-terminating,
    \item $\pi$ is a strong cyclic solution of $P$ and $\pi$ is $P$-terminating.
  \end{enumerate}
\end{theorem}
\begin{proof}
  \textbf{($1 \Leftrightarrow 2$)}
  Assume that $\pi$ solves $Q$. If there is $\pi$-trajectory connecting an initial state
  with state $s$, there must be $\pi$-trajectory connecting $s$ with a goal state.
  Otherwise, $\pi$ would not be a solution for $Q$.
  Likewise, if $\tau$ is an infinite $\pi$-trajectory in $Q$, then $\tau$ does not
  reach a goal state and thus $\pi$ would not solve $Q$.
  For the converse direction, assume that $\pi$ is a strong cyclic solution for $Q$
  and that $\pi$ is $Q$-terminating, and suppose that $\pi$ does not solve $Q$.
  Then, there is a maximal $\pi$-trajectory $\tau$ in $Q$ that does not reach a goal state.
  It cannot be the case that $\tau$ ends in a state $s$ where $\pi(s)$ is undefined
  or non-applicable as $\pi$ then would not be a strong cyclic solution for $Q$.
  Hence, $\tau$ must be infinite but this contradicts the assumption that $\pi$ is $Q$-terminating.

  \textbf{($2 \Leftrightarrow 3$)}
  By Theorem~\ref{thm:strong-cyclic}, $\pi$ is a strong cyclic solution for $Q$ iff
  it is strong cyclic solution for $P$.
  By Theorem~\ref{thm:termination2}, $\pi$ is $Q$-terminating iff it is $P$-terminating.
\end{proof}

\section{Checking Termination with  \sieve}

\sieve is the  procedure introduced by \citeay{sid:aaai2011} to test whether a policy terminates.
It runs in time that is polynomial in the number of states of the FOND problem reached by the policy.\footnote{Our account
  of termination differs from the one of \citeay{sid:aaai2011} in two main aspects.
  First, our notion of termination is articulated independently of the algorithm for checking termination.
  Second, our version of the algorithm is developed and applied to QNPs that involve both numerical and boolean variables.}
  %In \shortcite{sid:aaai2011},
  %the notion of termination is not articulated independently of the algorithm, so our account departs
  %from the one in that paper. A second difference is that our version of the  algorithm is developed
  %and applied to QNPs that involve both numerical and boolean variables.}
For this, the algorithm takes as input  a policy  graph $\G(P,\pi) = \tup{V,E}$
constructed from the FOND problem $P=T_D(Q)$ and a strong cyclic policy $\pi$ for $P$.
The  nodes in the policy  graph are the states $\bar s$ in the state model $\S(P)$ that are reachable from the initial state
and the policy $\pi$, and the  directed  edges in $E$ are the pairs  $(\bar s,\bar s')$ for  $\bar s' \in F(a,\bar s)$ and $\pi(\bar s)=a$. These edges are
labeled with the action $a$. The algorithm  iteratively removes edges from the graph $\G(P,\pi)$  until
the graph becomes acyclic or no additional edge can be removed.
% In the first case, the policy $\pi$ is terminating in $P$, otherwise, it is not.

For incrementally  removing edges from the graph, \sieve
identifies first  its  \emph{strongly connected components}
by a single depth-first search traversal, following
Tarjan's algorithm \shortcite{tarjan:sccs}. A strongly connected component (SCC)
is a partition of the nodes of the graph  such that if a node $\bar s$
belongs to a partition,  any node $\bar s'$ that can be reached from $\bar s$
and that can reach $\bar s$ back in the graph, is placed in the same partition as $\bar s$.

The algorithm then picks a variable $X$ and a SCC such that
the variable $X$ is decremented but not incremented in the SCC.
That is, there must be  a state $\bar s$ in the SCC such that $\pi(\bar s)$
is a $Dec(X)$ action and no $\pi(\bar s')$ is an $Inc(X)$ action for any $\bar s'$ in the SCC.
The algorithm then removes all the edges $(\bar s,\bar s')$ in the SCC such that  $\pi(\bar s)$ is a $Dec(X)$ action.
We abbreviate this by saying that \emph{variable $X$ is removed from the SCC},
which means that the edges associated with $Dec(X)$ actions are removed.
Following the edge removals, the SCCs must be recomputed and the process
is repeated until the graph becomes acyclic or no more edges can be removed
in this manner.
The result is that:

\begin{theorem}
  \label{thm:sieve}
  Let $Q$ be a QNP.
  A policy $\pi$ for the FOND problem $P=T_D(Q)$ is $P$-terminating iff
  \sieve reduces the policy graph $\G(P,\pi)$ to an acyclic graph.
\end{theorem}
\begin{proof}
  We show the contrapositive of the two implications that make up the equivalence.
  %For the first direction, let us assume that $\pi$ is $P$-terminating and
  %suppose that \sieve terminates with a cyclic graph $\G$.
  First, let us assume that \sieve terminates with a cyclic graph $\G$ and, thus,
  let $C$ be a non-trivial SCC in the graph $\G$.
  Since every state in $\G$ is reachable from the initial state by $\pi$ and
  since every state in $C$ is also reachable by $\pi$ from any other state in $C$,
  there is a $\pi$-trajectory $\bar\tau$ in $P$ of the form
  %Then, there must be a $\pi$-trajectory $\bar\tau$ in $P$ of the form
  $\bar s_0, \ldots, [\bar s_i, \ldots, \bar s_m]^*$
  where the recurrent states $\bar s_i, \ldots, \bar s_m$ in $\bar\tau$ are
  \emph{exactly} all the states in $C$. %a SCC $C$ of $\G$.
  Observe that \sieve is unable to remove any further edge from $\G$.
  If no variable is decremented in the loop, $\bar\tau$ is not $P$-terminating by definition.
  Similarly, if some variable $X$ is decremented in the loop (i.e., $C$),
  $X$ is also incremented in the loop and the trajectory is again not $P$-terminating.
  %Therefore, $C$ contains edges (actions) that decrement variables, and
  %for any variable $X$ decremented in $C$, there is another edge in $C$
  %that increments $X$. Therefore, the trajectory is non $P$-terminating.
  %By the assumption, there must be a variable $X$ such that $\pi(s)$ is a
  %$Dec(X)$ action for some state $\bar s$ in the loop, and $\pi(s')$ is not
  %an $Inc(X)$ action for any of the states $s'$ in the loop.
  %But then \sieve should have removed the variable $X$ from $C$, or any other
  %similar variable, breaking the SCC $C$ into smaller components.

  For the other implication, let us assume that $\pi$ is \emph{not} $P$-terminating.
  Then, there must be an infinite $\pi$-trajectory $\bar\tau$ in $P$ of the form
  $\bar s_0, \ldots, [\bar s_i, \ldots, \bar s_m]^*$ where every variable
  that is decremented by an  action $\pi(\bar s_j)$, $j\in[i,m]$,
  is incremented  by another  action $\pi(\bar s_j)$, $j\in[i,m]$.
  We want to show that \sieve terminates with a graph that has one SCC that
  includes all the states in the loop.
  Indeed, initially, all the states in the loop must be in one component $C$ as
  they are all reachable from each other.
  Observe that edges that correspond to actions that do not decrement variables are not removed by \sieve.
  Hence, if the loop does not decrement any variable $X$, the states in the loop cannot be
  separated into different SCCs.
  On the other hand, \sieve cannot remove a variable $X$ from the loop as the component $C$
  features actions that increment and decrement $X$. Therefore, as before, the states in the loop
  stay together within an SCC for  the whole execution of \sieve.
\end{proof}

\noindent The  \sieve  procedure,  slightly reformulated from \citeay{sid:aaai2011}, is
depicted in Figure~\ref{alg:sieve2}.

\begin{algorithm}[t]
  \SetKw{Break}{break}
  \SetKw{Continue}{continue}
  \DontPrintSemicolon
  %\par\noindent\rule{\columnwidth}{0.8pt}
  \sieve(Graph $\G=\G(P,\pi)$): \\
  %\Begin{
  \Repeat{$\G$ is acyclic (terminating) or there is no SCC $C$ and variable $X$ to choose (non-terminating)}{
    Compute the strongly connected components (SCC)  of $\G$.\;
    \BlankLine
    Choose an SCC $C$ and a variable $X$ that is decreased in $C$ but is not increased in $C$;\; i.e.,
    for some $\bar s$ in $C$, $\pi(\bar s)$ is a $Dec(X)$ action, and for no $\bar s$ in $C$, $\pi(\bar s)$ is an $Inc(X)$ action.
    %   any variable decremented by $e$ is not incremented by any other edge in $C$\;
    %-- Choose SCC $C$ with no predecessor SCC in the DAG and edge $e\in C$ such that
    %   any variable decremented by $e$ is not incremented by any other edge in $C$\;
    \BlankLine
    Remove the edges $(\bar s,\bar s')$ such that $\bar s$ and $\bar s'$ are in $C$, and $\pi(\bar s)$ is a $Dec(X)$ action.\;
    % $e=$ from $C$  \; %and all other edges with the same source and labeled with the action $a$\;
  %}
  %\BlankLine
  %$\pi$ accepts $\pi$ if $\G$ is  acyclic
  }
  \caption{\sieve procedure for testing whether policy $\pi$ for FOND problem  $P=T_D(Q)$  terminates \shortcite{sid:aaai2011}.}
  %The algorithm runs in polynomial time in the size of the graph $\G(P,\pi)$. % Slightly reformulated from
  %%\CHECK{**** Condition: is it enough one variable that is not incremented by any other edge, or we need that all decremented variables? ****}
  \label{alg:sieve2}
\end{algorithm}

\begin{example}
  The policy $\pi$ for $Q_{nest}$ above is given by the rules:
  \begin{alignat}{1}
    &\text{if $\GT{X}$ and $\EQ{Y}$, then do $a$} \,, \\
    &\text{if $\GT{X}$ and $\GT{Y}$, then do $b$}
  \end{alignat}
  where recall that $Q_{nest}=\tup{F,V,I,O,G}$ with $F=\emptyset$, $V=\{X,Y\}$,
  $I=\{\GT{X},\GT{Y}\}$, $G=\{\EQ{X}\}$, and $O=\{a,b\}$ where
  $a=\abst{\GT{X}, \EQ{Y}}{\DEC{X}, \INC{Y}}$ and $b=\abst{\GT{Y}}{\DEC{Y}}$.
  The policy decrements $Y$ using the action $b$ until the action $a$ that
  decreases $X$ and increases $Y$ can be applied. The process is repeated until $\EQ{X}$.
  The nested loops in the policy graph $\G(P,\pi)$ are shown in Figure~\ref{fig:qnest}.
  The policy graph contains three states: the leftmost one is the initial state, and
  the rightmost one is a goal state.
  The two states on the left are reachable from each other, and hence, define a strongly
  connected component (SCC). In this SCC, the $Dec(X)$ edges are removed by \sieve because
  $X$ is not increased anywhere. Once this is done, the $Dec(Y)$ edges are removed by
  \sieve because the edges associated with $Inc(Y)$ effects are gone.
  The resulting graph is acyclic establishing thus that the policy $\pi$ terminates
  in $P=T_D(Q_{nest})$.
\end{example}

\medskip

Using \sieve it is easy to see that the problem of checking the existence of plans
% for QNPs \emph{without propositional fluents} can be decided in exponential space:
for QNPs can be decided in exponential space:
% This result is
% not explicitly given by \citeauthor{sid:aaai2011}
% but it is a straightforward consequence of their results:

\begin{theorem}[\shortciteauthor{sid:aaai2011}, 2011] %Srivastava et al., 2011] %\citeay{sid:aaai2011}]
  \label{thm:plan-existence:expspace}
  Deciding plan existence for QNPs  is in \textup{EXPSPACE}.
\end{theorem}
\begin{proof}
  Let $Q=\tup{F,V,I,O,G}$ be a QNP. The number of boolean states for $Q$ is exponential
  in the number of fluents and variables; i.e., $|F|+|V|$. A policy $\pi$ for $Q$ can be described
  in exponential space as a mapping from boolean states into actions. A brute-force algorithm enumerates
  all policies one by one using exponential space. Each policy is tested for strong cyclicity and termination.
  The former is a straightforward test in a graph while the latter is done with \sieve.
  If the policy is strong cyclic and terminating, $Q$ is accepted.
  Otherwise, if no policy is found to be strong cyclic and terminating, $Q$ is rejected.
  Since testing strong cyclicity and running \sieve both require polynomial time in the size of the input
  policy, the whole algorithm can be implemented in space that is exponential in the size of $Q$.
\end{proof}

Below we improve this bound and show exponential time (EXP) solvability through a more
complex translation of QNPs into FOND problems that is also polynomial. The novelty of
the new translation is that  the strong-cyclic policies of the resulting FOND problems
do not need to be checked for termination.
QNPs are thus \emph{fully reduced} to FOND problems.
Since FOND problems can be reduced to QNPs as well, \emph{we will show indeed that FOND
problems and QNPs have the same expressive power}, and the complexity of plan existence
for FOND problems is known to be EXP-Complete \shortcite{littman:fond,rintanen:po}, then
these reductions show the EXP-Completeness of the plan existence decision problem for QNPs.
In addition to the establishment of novel theoretical results, these reductions are also
of practical importance as they permit the computation of solutions; once a QNP is reduced
to FOND, a solution for the QNP can be recovered in linear time from a solution to the
FOND problem, and the same for the reduction from FOND problems into QNPs.

The distinction between the classes EXP and EXPSPACE is important. EXPSPACE contains the
(decision) problems that can be solved with  Turing machines (TMs) that operate in
exponential space (as a function of the input size), yet such TMs may in fact run in
doubly exponential time as the running time is bounded, in the worst case, by an
exponential function of the bound in space \shortcite{sipser:book}.
On the other hand, EXP comprises the problems that can be solved with TMs that run
in exponential time. The difference is analogous to the difference between the
classes P (polynomial time) and PSPACE (polynomial space).

\begin{figure}[t]
  \centering
  \begin{tikzpicture}[thick,>={Stealth[inset=2pt,length=8pt,angle'=33,round]},font={\normalsize},qs/.style={draw=black,fill=gray!20!white},init/.style={qs,fill=yellow!50!white},goal/.style={qs,fill=green!50!white}]%, show background rectangle, show background grid]
    \node[init] (A) at  (0,0) { $\GT{X}, \GT{Y}$ };
    \node[qs]   (B) at  (5,0) { $\GT{X}, \EQ{Y}$ };
    \node[goal] (C) at (10,0) { $\EQ{X}, \GT{Y}$ };
    \path[->] (A) edge[out=140,in=40,looseness=4] node[above,xshift=1] { $b: \DEC{Y}$ } (A);
    \path[->] (A) edge[transform canvas={}] node[above,yshift=-1] { $b: \DEC{Y}$ } (B);
    \path[->] (B) edge[out=220,in=320,looseness=0.7] node[below,yshift=0] { $a: \DEC{X}, \INC{Y}$ } (A);
    \path[->] (B) edge[transform canvas={}] node[above,yshift=-2] { $a: \DEC{X}, \INC{Y}$ } (C);
  \end{tikzpicture}
  \caption{%
    Testing termination with \sieve. Policy graph $\G(P,\pi)$ for the FOND problem
    $P=T_D(Q_{nest})$ and policy $\pi$, from the example in text, containing three states:
    the leftmost is the initial state, and the rightmost is a goal state.
    The two states on the left are reachable from each other, and hence, define
    a strongly connected component (SCC). In this SCC, the $Dec(X)$ edges are removed
    by \sieve because $X$ is not increased anywhere. Once this is done, the $Dec(Y)$
    edges are removed because the edges associated with $Inc(Y)$ effects have been eliminated.
    The resulting graph is acyclic and hence $\pi$ terminates in $P$.
  }
  \label{fig:qnest}
\end{figure}
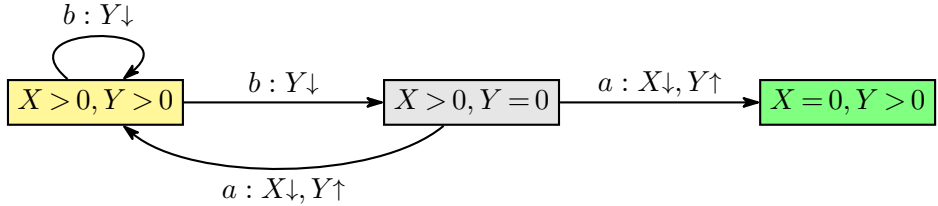

\section{First Reduction: FOND problems into QNPs}

The first reduction that we introduce is from FOND problems into QNPs.
It is a non-trivial reduction yet simpler than the inverse reduction from QNPs into FOND problems.
The main obstacle to overcome is that the non-deterministic effects in
FOND problems are over boolean variables, while those in QNPs are only
on numerical variables through  decrements.
Another important obstacle is that strong cyclic solutions in QNPs are
not QNP solutions unless they are terminating.

Let $P=\tup{F,I,O,G}$ be a FOND problem and let us denote the non-deterministic
effects of action $a$ as $E^a_1\,|\,E^a_2\,|\,\cdots\,|\,E^a_{k_a}$ where each
$E^a_i$ is a set (conjunction) of $F$-literals, and $k_a$ denotes the number of
non-deterministic effects of the action $a$.
If $k_a=1$, the action $a$ is deterministic, else it is non-deterministic.
For simplicity, we assume that the set of effects $\{E^a_i\}_i$ for action $a$
aggregates all the multiple non-deterministic effects in the description of $a$
in $P$, and the reduction below is presented under this assumption.
Afterwards, we discuss how to handle in polynomial time FOND problems whose transitions
are factorized.

We map $P$ into a QNP $Q=\tup{F',V',I',O',G'}$ that extends $P$ with numerical variables
$V'=\{X\}\cup\{Y_{a,i} : a\in O, 1\leq i\leq k_a\}$, extra boolean variables, and extra
actions; i.e., $F \subseteq  F'$, $I\subseteq I'$, $O \subseteq O'$, and $G'=G$.

The heart of the  reduction lies in the way in which the non-deterministic effects
of each action are captured in $Q$.
For this, the collection of non-deterministic effects of the action $a$ are replaced by
an $Inc(X)$ action, for the unique numerical variable $X$, followed by a \emph{fixed loop}
where the variable $X$ is decremented until becoming zero.
The alternative effects $E_i^a$ are then triggered when $\EQ{X}$ becomes true in
the corresponding part of the loop.
Intuitively, each non-deterministic action $a$ becomes a ``wheel of fortune'' that must
be spun to select the non-deterministic effect to be applied.
The increments and decrements of the extra variables $Y_{a,i}$ ensure that the strong
cyclic policies $\pi$ for $P$, and only those, induce policies $\pi'$ for $Q$ that are
terminating.
The fixed loop sequence associated with action $a$ performs the following steps, that
are implemented with the help of new auxiliary actions and propositional symbols:
\begin{enumerate}[1.]
  \item $Inc(X)$ (implemented by modified action $a$),
  \item $Dec(X)$ (implemented by new action $Spin$),
  \item If $X=0$, apply the effects in $E^a_1$, increment $Y_{a,1}$, decrement $Y_{a,j}$, $j\neq 1$,
    and break loop (implemented by new actions $Prep(a,i)$ and  $Exit(a,i)$ for $i=1$),
  \item $Dec(X)$ (implemented by new action $Next(a,i)$ for $i=1$)
  \item If $X=0$, apply the effects in $E^a_2$, increment $Y_{a,2}$, decrement $Y_{a,j}$, $j\neq 2$,
    and break loop (implemented by new actions $Prep(a,i)$ and  $Exit(a,i)$ for $i=2$),
  \item $Dec(X)$ (implemented by the new action $Next(a,i)$ for $i=2$)
  \item[\vdots]
  \item If $X=0$, apply the effects in $E^a_{k_a}$, increment $Y_{a,{k_a}}$, decrement $Y_{a,j}$, $j\neq k_a$,
    and break loop (implemented by new actions $Prep(a,i)$ and  $Exit(a,i)$ for $i=k_a$),
  \item $Dec(X)$ and go back to 3 (implemented by new action $Loop(a)$)
\end{enumerate}
If, throughout the loop, some variable $Y_{a,j}$ becomes zero, the next action to apply
is forced to be a new action that achieves the goal thus ending the execution.
This mechanism takes care of the non-fair trajectories that may exist in $P$.

To show that the resulting mapping is indeed a reduction (i.e., the mapping is sound and
complete), two things need to be established: that a policy $\pi$ that solves $Q$ induces
a policy $\pi'$ that solves $P$ (soundness), and  vice versa, that a policy $\pi'$ that
solves the FOND problem $P$ induces a policy $\pi$ that solves $Q$ (completeness).

The fixed sequence of effects for each action $a$ in $Q$ is implemented using the
following additional boolean variables:
\begin{enumerate}[$\bullet$]
  \item A boolean $normal$ that is false when the sequence is entered for some action $a$ and made
    true when the loop is exited.
  \item A boolean $ex(a)$ to express that the sequence for the action $a$ is being executed.
    The reduction is such that the atoms in $\{normal\}\cup\{ex(a):a\in O\}$ are pairwise mutex and one of them is always true.
  \item A counter from $0$ to $K$ encoded using mutex atoms $cnt(\ell)$, $\ell=0,1,\ldots,K+1$, that is set to $1$
    when the loop is entered and re-entered (step 8 above) and where $K$ is the maximum number of non-deterministic
    outcomes of any action in $P$. (An engineering trick that saves an extra boolean permits $cnt(0)$ and $cnt(i)$ to be both true after executing the action $Prep(a,i)$ below.)
\end{enumerate}
The actions that implement the fixed loop sequence (i.e., spin the wheel)  are the following:
\begin{enumerate}[$\bullet$]
  \item $Spin=\abst{\neg normal, cnt(0), \GT{X}}{\neg cnt(0), cnt(1), \DEC{X}}$.
  \item $Next(a,i)=\abst{ex(a), cnt(i), \GT{X}}{\neg cnt(i), cnt(1+i), \DEC{X}}$ to advance along the fixed loop sequence while $\GT{X}$ and decrementing $X$.
  \item $Loop(a)=\abst{ex(a), cnt(k_a), \GT{X}}{\neg cnt(k_a), cnt(1), \DEC{X}}$ to start a new iteration of the loop.
\end{enumerate}
These are the only actions that decrease the variable $X$, while the (modified) non-deterministic actions
being the only actions that increase $X$ (see below). Once $X$ becomes zero, the non-deterministic effect to apply is
determined by the value of the counter. % (encoded with the atoms $cnt(i)$).
%Extra actions and variables are used to take care of the possible non-fair trajectories in $P$.
The actions that apply the selected effect and capture the non-fair trajectories are:
\begin{enumerate}[$\bullet$]
  \item $Prep(a,i)=\abst{ex(a), cnt(i), \neg cnt(0), \EQ{X}, \GT{Y_{a,j}} }{cnt(0), \INC{Y_{a,i}}, \{ \DEC{Y_{a,j}}\}_{j\neq i} }$
    where the precondition $\GT{Y_{a,j}}$ is for all $j$, and the effect $\DEC{Y_{a,j}}$ for all $j\neq i$.
  \item $Exit(a,i)=\abst{ex(a), cnt(i), cnt(0), \EQ{X}, \GT{Y_{a,j}}}{\neg ex(a), \neg cnt(i), normal, E^a_i}$
    where the precondition $\GT{Y_{a,j}}$ is for all $j$.
  \item $Fin(a,j)=\abst{ex(a), \EQ{X}, \EQ{Y_{a,j}}}{\neg ex(a), \{\neg cnt(i)\}_{i\geq1}, normal, G}$ that is forced when $\EQ{Y_{a,i}}$ and reaches the goal $G$,
    where the effect $\neg cnt(i)$ is for all positive $i$.
\end{enumerate}
The only actions that affect the variables $Y_{a,j}$ are the $Prep(a,i)$ actions. After applying such
an action, some (one or more) variables $Y_{a,j}$ may become zero. If so, $Fin(a,j)$ becomes the only applicable
action and the execution terminates as the goal is reached.
If no such variable becomes zero, $Exit(a,i)$ becomes forced which applies the effect $E^a_i$ and restores ``normal'' operation.
%Observe how the non-fair trajectories, that do not need to be taken care of when finding a solution for $P$, are handled.
%By definition, a non-fair trajectory is one in which an action $a$ is applied infinitely often and some of its
%outcomes, say the $j$-th outcome,  is ``starved''. In such a case,  the $Prep(a,j)$ that increments $Y_{a,j}$
%is applied a finite number of times, yet other $Prep(a,i)$ actions that decrement $Y_{a,j}$ are applied
%infinitely often.
%As a result, the variable $Y_{a,j}$ becomes zero in such a trajectory forcing the action $Fin(a,j)$ to be applied.

The initial state of the QNP includes the atoms $normal$ and $cnt(0)$.
Deterministic actions in $P$ ``pass directly'' into the QNP $Q$ with these
two atoms added as extra preconditions.
Non-deterministic actions $a$, however, are handled differently by replacing their
effects $E^a_1\,|\,E^a_2\,|\,\cdots\,|\,E^a_{k_a}$ by deterministic effects
$\{\neg normal, ex(a),\INC{X}\}$ after which the only applicable action would be
the $Spin$ action from above.
The idea of the  construction is illustrated in Figure~\ref{fig:compilation-fond:2}.

\begin{figure}[t]
  \centering
  \begin{tikzpicture}[thick,>={Stealth[inset=2pt,length=8pt,angle'=33,round]},font={\footnotesize},qs/.style={draw=black,fill=gray!20!white},init/.style={qs,fill=yellow!50!white},goal/.style={qs,fill=green!50!white}]%, show background rectangle, show background grid]
    \node[qs]      (a)  at   (0,2) { $cnt(0),\EQ{X},\GT{Y_j}$ };
    \node[qs]  (entry)  at   (0,0) { $cnt(0),\GT{X},\GT{Y_j}$ };
    \node[qs] (count1)  at  (0,-2) { $cnt(1),\GT{X},\GT{Y_j}$ };
    \node[qs]  (exit1)  at  (6,-2) { $cnt(1),\EQ{X},\GT{Y_j}$ };
    \node[qs] (count2)  at  (0,-4) { $cnt(2),\GT{X},\GT{Y_j}$ };
    \node[qs]  (exit2)  at  (6,-4) { $cnt(2),\EQ{X},\GT{Y_j}$ };
    %\node[qs] (countkm) at  (0,-6) { $cnt(k{-}1),\GT{X},\GT{Y_j}$ };
    \node[qs] (countk)  at  (0,-7) { $cnt(k),\GT{X},\GT{Y_j}$ };
    \node[qs]  (exitk)  at  (6,-7) { $cnt(k),\EQ{X},\GT{Y_j}$ };

    \path[->]      (a) edge[transform canvas={xshift=-20}] node[right,yshift=0]  { (modified) $a: \INC{X}$ }     (entry);
    \path[->]  (entry) edge[transform canvas={xshift=-20}] node[right,yshift=0]  { $Spin: \DEC{X}$ }     (count1);
    \path[->]  (entry) edge[transform canvas={xshift=0}]   node[sloped,yshift=7] { $Spin: \DEC{X}$ }      (exit1);
    \path[->] (count1) edge[transform canvas={xshift=-20}] node[right,yshift=-6] { $Next(a,1): \DEC{X}$ } (count2);
    \path[->] (count1) edge[transform canvas={xshift=0}]   node[sloped,yshift=7] { $Next(a,1): \DEC{X}$ }  (exit2);

    % edges from exit nodes
    \path[->]  (exit1) edge[dashed] node[above,yshift=-1] { $Prep(a,1): \INC{Y_1}, \{\DEC{Y_j}\}_{j\neq 1}$ } (12,-2);
    \path[->]  (exit2) edge[dashed] node[above,yshift=-1] { $Prep(a,2): \INC{Y_2}, \{\DEC{Y_j}\}_{j\neq 2}$ } (12,-4);
    \path[->]  (exitk) edge[dashed] node[above,yshift=-1] { $Prep(a,k): \INC{Y_k}, \{\DEC{Y_j}\}_{j\neq k}$ } (12,-7);

    % dotted broken edges
    %\path[-]    (count2) edge[dotted,transform canvas={xshift=0}] ($(count2)!0.30!(6,-7)$);
    \path[-]    (count2) edge[dotted,transform canvas={xshift=-20}] (0,-4.80);

    % edges into bottom nodes
    \path[->]   (0,-5.6) edge[transform canvas={xshift=-20}]        node[right,yshift=-2] { $Next(a,k{-}1): \DEC{X}$ } (countk);
    \path[->] (1.6,-5.6) edge[transform canvas={xshift=-5}]         node[sloped,yshift=7] { $Next(a,k{-}1): \DEC{X}$ }  (exitk);

    % loop
    \path[-] (-2.0,-7) edge[out=0,in=180] (countk);
    \path[-] (-2.0,-7) edge[] node [sloped,yshift=6] { $Loop(a): \DEC{X}$ } (-2.0,-2);
    \path[->] (-2.0,-2) edge[out=0,in=180] (count1);
  \end{tikzpicture}
  \caption{%
    Encoding the non-deterministic boolean effects $E_1\,|\,\cdots\,|\,E_k$ of an action $a$
    in the FOND problem $P$ as a fixed sequence of effects that loops while decrementing the
    variable $X$ in the QNP $Q=R(P)$.
    The variable $X$ implements the loop while the counter $cnt(i)$ determines which effect
    $E_i$ obtains when $X$ becomes zero. When $\EQ{X}$ and $cnt(i)$ hold, $Prep(a,i)$ is applied
    whose function is to increment $Y_i=Y_{a,i}$ and decrement the other $Y_j$ variables; if no
    such variable becomes zero, the effect $E_i$ is applied with the action $Exit(a,i)$ (not shown),
    otherwise $Fin(a,j)$ is applied when $Y_j$ becomes zero (not shown).
    The variables $Y_i$  are used to map \emph{unfair} trajectories in $P$ into
    goal-reaching trajectories in $Q$, while forcing  solutions of $Q$ to induce solutions for $P$.
    Indeed, if one unfair trajectory contains the action $a$ infinitely often but neglects
    (starves) the effect $E_i$, the variable $Y_i$ eventually becomes zero, since the only
    action that increments it is $Prep(a,i)$, and then the trajectory forcibly applies the
    action $Fin(a,i)$ that terminates the execution by reaching the goal
    (see text for details).
  }
  \label{fig:compilation-fond:2}
\end{figure}
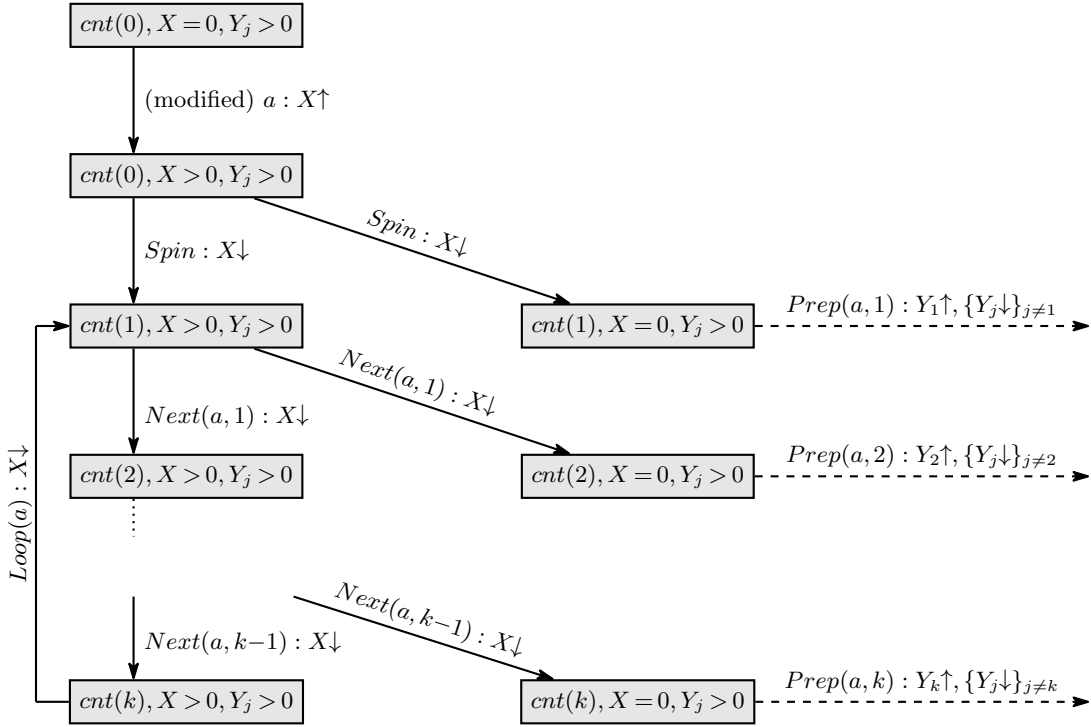

\begin{definition}
  \label{def:reduction:2:fond->qnp}
  Let $P=\tup{F,I,O,G}$ be a FOND problem.
  The reduction $R$ maps $P$ into the QNP $Q=\tup{F',V',I',O',G'}$ given by
  \begin{enumerate}[1.]
    \item $F'=F \cup \{normal\} \cup \{ ex(a) : a \in O \} \cup \{ cnt(\ell) : \ell\in[0,K] \}$ where $K=\max_{a\in O} k_a$,
    \item $V'=\{X\} \cup \{Y_{a,i} : a \in O, i\in[1,k_a] \}$,
    \item $I' = I \cup \{normal,cnt(0)\} \cup \{\EQ{X}\} \cup \{ \GT{Y_{a,i}} : a \in O, i\in[1,k_a] \}$,
    \item $O' = O \cup \{Spin\} \cup \{ \mathcal{A}(a,i) : \mathcal{A}\in\{Prep,Exit,Fin,Next\}, a \in O, i\in[1,k_a] \} \cup \{ Loop(a) : a \in O \}$,
    \item $G' = G$.
  \end{enumerate}
\end{definition}

For stating the formal properties of the reduction, we need to specify how a policy $\pi$ for $Q$
induces a policy $\pi'$ for $P$, and vice versa, as $Q$ involves extra variables and actions.

Let us split the \emph{non-goal states} in $Q$ into two sets: the normal states, where the booleans
$normal$ and $cnt(0)$ are true, and the rest of non-goal states where $normal$ or $cnt(0)$ is false.
(Observe that there cannot be a normal state where some variable $Y_{a,j}$ is zero as when that
happens, the loop must exit with $Fin(a,j)$ that reaches a goal state.)
In the normal states, the (qualitative) value of all the extra variables is the same: all the extra
boolean variables are false except for $normal$ and $cnt(0)$ that are true, $\EQ{X}$, and
$\GT{Y_{a,i}}$ for $a\in O$ and $i=1,2,\ldots,k_a$.
Hence, any policy for $P$ can be converted into a policy for $Q$, and vice versa. For example,
a policy $\pi$ for $P$ translates directly into a policy for $Q$ on normal states, where the
actions for non-normal states is determined (the only real choice is when more that one variable
$Y_{a,j}$ is zero, but then any choice of the $Fin(a,j)$ actions is equivalent as such actions yield goal states).
Conversely, any policy for $Q$ induces a unique policy for $P$ as each non-goal state in $P$
is associated with normal states in $Q$ that get the same action from the policy.
%
%Taking advantage of this relation and leaving the part of the normal states over $Q$ that
%is fixed implicit, we obtain that a policy $\pi$ for $P$ represents a unique policy $\pi$
%over the normal states in $Q=R(P)$, and vice versa, a policy $\pi$ over the normal states
%represents a unique policy for $P$.
%
%The policy over the non-normal states $t$ in $Q$ is determined as for each such state there
%is only one (new) action applicable; i.e., non-normal states $t$ are mapped to one and only
%one of the actions in $O'\setminus O$ according to the action precondition that is true in
%$t$; observe that the preconditions of these actions are mutually exclusive.
%
The main properties of the reduction are expressed as follows:

\begin{theorem}[Reduction FOND to QNP]
  \label{thm:reduction:fond->qnp}
  The mapping $R$ is a polynomial-time reduction from FOND problems into QNPs.
  That is, a FOND problem $P$ has a solution iff the QNP $Q=R(P)$ has a solution.
  Moreover, a solution for $P$ (resp.\ $Q$) can be recovered in polynomial time
  from a solution for $Q$ (resp.\ $P$).
\end{theorem}
\begin{proof}
  It is straightforward to see that $R$ is computable in time that is polynomial in
  the size of the input $P$ and the maximum value $k_a$ for any action $a$.
  (See below for a discussion on how to deal with FOND problems that may have
  multiple non-deterministic effects per action.)
  Likewise, as explained above, policies for $P$ map into policies for $Q$,
  and vice versa.
  %In the following we show that $R$ is indeed a reduction.

  We first show how to map infinite fair trajectories in $P$ into trajectories in $Q$ that visits normal states infinitely often, and vice versa.
  Indeed, for an infinite trajectory $\tau$ in $P$, we can construct a corresponding trajectory $\tau'$ in $Q$
  by augmenting the states in $\tau$ to give value to the extra booleans and numerical variables in $Q$, and
  by ``inserting'' sub-trajectories over non-normal states for the triplets $(s,a,s')$ in $\tau$ for non-deterministic actions $a$.
  Indeed, for such a triplet, it is enough to determine the effect $E^a_i$ that produces the successor state $s'$ in order
  to obtain a finite sequence of non-normal states that map $s$ into $s'$ in $Q$.
  For the converse, if $\tau'$ is a trajectory in $Q$ that visits normal states infinitely often, then $\tau'$ does not contain an action
  $Fin(a,i)$. Thus, we can construct a trajectory over $P$ by just projecting $\tau'$ over its normal states.
  This trajectory is fair since if some effect $E^a_i$ is starved from an action $a$ that is
  applied infinitely often, the variable $Y_{a,i}$ would become zero forcing the application of $Fin(a,i)$.

  For showing the first implication, let us consider a policy $\pi$ for $P$. We want to show that the policy $\pi'$ obtained from $\pi$ solves $Q$.
  Indeed, suppose that $\pi'$ does not solve $Q$. Then, $\pi'$ is either not strong cyclic or non terminating in $Q$.
  The former is impossible since $\pi$ connects each reachable normal state to a goal state.
  Thus, $\pi'$ is non terminating in $Q$ and there is a non-terminating trajectory $\bar s_0,\ldots[\bar s_i,\ldots,\bar s_m]^*$ in $T_D(Q)$.
  The loop increments and decrements the variable $X$, and all the variables $Y_{a,j}$ for the actions $a$ that are applied in the loop,
  and thus it must contain some normal state.
  Therefore, we can construct an infinite fair $\pi$-trajectory in $P$ that is non-goal reaching contradicting the assumption that $\pi$ solves $P$.

  For showing the second implication, let $\pi'$ be a policy for $Q$ and let us suppose that the policy $\pi$ induced from $\pi'$ does not solve $P$.
  Then, either $\pi'$ is not defined over its reachable states or there is an infinite fair $\pi$-trajectory in $P$, one that does not reach the goal.
  The first case is impossible since $\pi'$ solves $Q$. For the second case, if such trajectory exists, then
  there is a $\pi'$-trajectory $\tau$ in $Q$ that visits normal states infinitely often.
  This trajectory is associated with a non-terminating trajectory in $T_D(Q)$ in contradiction with $\pi'$ being a solution for $Q$.
\end{proof}

Theorem~\ref{thm:reduction:fond->qnp} states that the strong cyclic policies for a FOND problem $P$
correspond exactly with the policies of the QNP $Q=R(P)$ that can be obtained from $P$ in polynomial time.
There is an analogous translation for computing the \emph{strong policies} of $P$; namely, the strong cyclic
policies that are actually acyclic (no state visited twice along any trajectory).
For this, let $R'$ be the translation that is like $R$ above but with the numerical variables $Y_{a,i}$ removed;
i.e., initial and goal conditions on $Y_{a,i}$ are removed from $R'(P)$ as well as the conditions
and effects on $Y_{a,i}$ in the actions $Prep(a,i)$ and $Exit(a,i)$, and with the actions $Fin(a,j)$ removed.

\begin{theorem}[Reduction Strong FOND to QNP]
  \label{thm:reduction:fond->qnp:strong}
  The mapping $R'$ is a polynomial time reduction from FOND problems with strong solutions
  into QNPs. That is, a FOND problem $P$ has a strong solution iff the QNP $Q=R'(P)$ has
  solution.
  Moreover, a strong solution for $P$ (resp.\ solution for $Q$) can be recovered in polynomial
  time from a solution for $Q$ (resp.\ strong solution for $P$).
\end{theorem}
\begin{proof}
  In the absence of the variables $Y_{a,i}$, the only solutions to $Q=R'(P)$ must be acyclic,
  as any cycle would involve increments and decrements of the single numerical variable $X$,
  and thus, would not be terminating. The rest of the proof follows from the correspondence
  laid out in the previous proof between the trajectories over the normal states in $Q$
  and the trajectories in $P$.
\end{proof}

%% Remark on factorized transition functions for FOND
Let us now consider the case when the FOND problem $P$ contains actions with
multiple non-deterministic effects.
Let $a$ be one such action, and let $n$ be the number of non-deterministic
effects in $a$, each one denoted by $E^j_1\,|\,\cdots\,|\,E^j_{k_j}$ where
$j\in[1,n]$ and $k_j$ is the number of outcomes for the $j$-th non-deterministic
effect of $a$.
By assumption, every choice of outcomes for the effects yields a \emph{consistent}
(aggregated) outcome for $a$ (cf.\ footnote \ref{foot:1}).
The easiest way to accommodate such actions is to ``preprocess'' $P$ by
replacing the action $a$ by the \emph{sequence} $\tup{a(1),a(2),\ldots,a(n)}$
of $n$ non-deterministic actions, each one featuring exactly one effect;
i.e., the effect of $a(j)$ is $E^j_1\,|\,\cdots\,|\,E^j_{k_j}$.
For this, the precondition of $a(1)$ is set to the precondition of $a$,
that of $a(j)$ to $\{seq(a),next(j)\}$, $j\in[2,n]$, and the precondition
of every other action is extended with $\neg seq(a)$.
Likewise, the effect of $a(1)$ is extended with $\{seq(a),next(2)\}$,
that of $a(j)$ with $\{\neg next(j),next(j+1)\}$, $j\in[2,n-1]$,
and that of $a(n)$ with $\{\neg next(n),\neg seq(a)\}$.
The new atoms $seq(a)$ and $next(j)$ denote that the sequence for $a$ is
being applied and the next action to apply in the sequence is $a(j)$
respectively. In this way, the FOND problem $P$ is converted in linear
time into an \emph{equivalent} FOND problem where each action has exactly
one non-deterministic effect.
In other words, we may assume without loss of generality that the FOND
problem $P$ is such that each action has at most one non-deterministic
effect since if this is not the case, $P$ can be converted into an
equivalent FOND problem $P'$ in linear time. By equivalent, we mean
that any solution $\pi$ for $P$ can be converted in polynomial time into
a solution $\pi'$ for $P'$, and vice versa.

\medskip

We have shown how strong cyclic and strong planning over a FOND problem $P$ translates
into QNPs:  in one case, including  the extra  variables
$X$ and $\{Y_{a,i}\}_{a,i}$ in place, in the second,  only $X$.
The translation with these variables offers the possibility of capturing more subtle
forms of fairness.
For example, if we just remove from the translation a variable $Y_{a,i}$
along with the effects on it, the resulting QNP would assume that all the outcomes $E_j$, $j\neq i$,
of action $a$ are fair (i.e., they cannot be skipped forever in a fair trajectory)
but that the outcome $E_i$ is not.
In other words, while in strong cyclic planning, all the non-deterministic actions are assumed
to be fair, and in strong planning, all of them to be unfair, in QNP planning,
it is possible to handle a combination of fair and unfair actions (as in dual FOND planning, Geffner \& Geffner, 2018), %\shortcite{geffner:fond-sat}),
as well as a combination of fair and unfair outcomes of the same action.

\section{Second Reduction: QNPs into FOND problems}

We have shown that FOND problems  can be reduced in polynomial time to QNPs.
We show now the other direction: QNPs  can be reduced in polynomial time to FOND problems.
The two  results imply a  new complexity result; namely, that QNPs have the same expressive
power as FOND problem and that the plan existence decision problem for both models have the
same complexity.
This second translation $T$ is more subtle than the first and unlike the direct translation
$T_D$ above, it is a full reduction which does not require termination tests.

The first attempt at such a translation was sketched by \citeay{bonet:ijcai2017} but the reduction
is buggy as  it is not always sound.
The intuition, however, is useful and we build on it. Basically, that reduction introduces boolean
variables $q_X$ that when set to true preclude increments of variable $X$, hence making the
decrements of $X$ ``fair''. The variable $q_X$ can be reset to false when the loop ``finishes'',
i.e., when  $\EQ{X}$ is true. This idea, however, does not fully avoid non-terminating loops and
hence, by itself, does not produce a sound reduction.\footnote{Consider a QNP $Q=\tup{F,V,I,O,G}$
  with a single numerical variable $X$ and four actions $a$, $b$, $c$, and $d$ that result in a
  loop where $a=\abst{p_1, \GT{X}}{\neg p_1, p_2, \DEC{X}}$, $b=\abst{p_2}{p_3, \neg p_2}$,
  $c=\abst{p_3, \GT{X}}{\DEC{X}}$, $d=\abst{p_3, \EQ{X}}{\neg p_3, p_1, \INC{X}}$.
  Let us assume that $I=\{p_1,\GT{X}\}$ and $G=\{\EQ{X},p_2\}$.
  There is a single policy $\pi$ for $Q$, as in all the (non-goal) states that can be reached from $I$,
  there is a single applicable action. This policy $\pi$ is strongly cyclic but is not terminating.
  The reason is that one of the trajectories determined by the policy is a non-terminating loop
  $s_0,a,s_1,b,s_2,c,s_3,d,s_0,\ldots$ where the single variable that is decremented ($X$) is
  also incremented, and where $\bar s_0=\{p_1,\GT{X}\}$, $\bar s_1=\{p_2,\GT{X}\}$, $\bar s_2=\{p_3,\GT{X}\}$, and $\bar s_3=\{p_3,\EQ{X}\}$.
  The FOND problem that results from the translation sketched by \citeay{bonet:ijcai2017} accepts
  this policy $\pi$ as a solution, which is incorrect. This happens because the variable $q_X$ can be
  set and reset an infinite number of times; indeed, right before and right after the action $c$
  in the loop, respectively.
  The new translation excludes such non-terminating loops via a stack and counters.
}
The \emph{new translation} replaces the $q_X$ variables by a \emph{bounded stack} that keeps
the variables $X$ that are being decremented \emph{in order}, and suitable \emph{counters}.
The new variables and actions enforce that solutions of the FOND problem $P=T(Q)$, unlike the
solutions of the direct translation $T_D(Q)$, are all terminating.
For this, the new translation introduces conditions that mirror those captured by the \sieve
procedure. In particular, for capturing policies that terminate, the variables are to be placed
on the stack following the order by which \sieve removes them.
% (recall that removing a variable $X$ in \sieve\
% means to remove the edges associated with $Dec(X)$ effects).

\subsection{Extra Variables and Actions}

The reduction $T(Q)$ introduces a bounded stack $\alpha$ where numerical variables from $V$ can
be pushed and popped, and bounded  counters $c(d)$, for $d=0, \ldots, |V|$ that are  associated
with the possible levels (depths)  $d=|\alpha|$  of the stack.
There is also a top counter $c_T$ that may only increase.
The stack starts empty and may grow to contain all the variables in $V$, but no variable can appear
in the stack more than once.
The stack is represented as growing from left to right; e.g., $\alpha X$ is the stack that results
of pushing the variable $X$ in the stack $\alpha$.
The $c$ counters start at $0$ and may grow up to a $Max$ number, that for completeness must be set
to $1+2^n$ where $n$ is the total number of boolean and numerical variables in $Q$.
In practice, $Max$ can be set to a much small number.\footnote{For structured policies that result in
  loops that can be entered and left through single entry and exit points, $Max$ is the bound on the
  number of consecutive loops (blocks), possibly with other loops nested, that the policy can generate
  at the same level.
}
In any case, the counters and the stack are captured in terms of a polynomial number of boolean
variables and the whole reduction $T(Q)$ is computed in polynomial time.
The state of the stack $\alpha$ is captured by the atoms $in(X)$, $depth(d)$, and $index(X,d)$
that represent whether $X$ is in the stack, the depth of the stack, and the depth at which $X$
is in the stack, respectively.
$X$ is the top element in the stack when $index(X,d)$ and $depth(d)$ are true (i.e., the stack
is $\alpha X$ and $|\alpha|=d-1$), and it is bottom element when $index(X,1)$ is true (i.e.,
the stack is $X$). The stack is empty when $depth(0)$ holds.

The extra actions in $P=T(Q)$ are those for pushing and popping variables to and from the stack,
and for advancing the top counter $c_T$.

%% (** how many $\alpha$'s possible? Each  prefix
%% of the $|V|$ permutations? That is, $|V| * |V|!$ stacks $\alpha$'s?
%% Exponential in the number of variables? Not poly??? **):

\begin{enumerate}[1.]
  \item \textbf{Actions $Push(X,d)$} for variable $X$ and depth $d\in[0,|V|-1]$ have preconditions
    $\neg in(X)$, $depth(d)$ and $c(d) < Max$, and effects:
    \begin{enumerate}[$a)$]
      \item $in(X)$, $index(X,d+1)$, $depth(d+1)$ and $\neg depth(d)$ to push $X$ and increase stack depth,
      \item $c(d) := c(d) + 1$ to increment counter for old level,
      \item $c(d+1):= 0$ to initialize counter for new level.
    \end{enumerate}
  \item \textbf{Actions $Pop(X,d)$} for variable $X$ and depth $d\in[1,|V|]$ have preconditions
    $in(X)$, $index(X,d)$ and $depth(d)$, and effects:
    \begin{enumerate}[$a)$]
      \item $\neg in(X)$, $\neg index(X,d)$, $\neg depth(d)$, $depth(d-1)$ to pop $X$ and decrease stack depth.
    \end{enumerate}
  \item \textbf{Action $Move$} advances the top counter $c_T$  by 1 when the stack is empty;
    i.e., it has preconditions $depth(0)$ and $c_T < Max$, and effect $c_T := c_T+1$.
\end{enumerate}

For simplicity, we assume that the language of our FOND problems $P$ makes room for the integer counters
$c(d)$, $d=0,\ldots,|V|$, that may be increased by 1, from $0$ up to a fixed number $Max$ and that may be
reset back to $0$. In the implementation, these counters are represented in terms of a linear number of
boolean variables.\footnote{In the actual encoding, the counters $c(d)$ are represented with $1+n$ atoms
  (bits), $b_i(d)$, $i=0,\ldots,n$, where $n$ is total number of variables in $Q$, i.e., $n=|F|+|V|$.
  A preconditions such as $c(d)<Max$ then translates into the precondition $\neg b_n(d)$; the least and
  most significant bits for $c(d)$ are $b_0(d)$ and $b_n(d)$ respectively.
  Increments of $c(d)$ by $1$ may be translated in two different ways, either by using conditional effects,
  or by increasing the number of actions. For the former, conditional effects of the form
  $b_0(d), \ldots b_{i-1}(d), \neg b_i(d) \rightarrow \neg b_0(d), \ldots, \neg b_{i-1}(d), b_i(d)$,
  for $i\in[0,n]$, are used. For the latter, each action $act$ that increases $c(d)$ is replaced by
  $n$ actions $act(i)$, $i\in[0,n]$, that are like $act$ but have the extra precondition
  $b_0(d), \ldots b_{i-1}(d), \neg b_i(d)$ and the extra effects $b_0(d), \ldots, \neg b_{i-1}(d), b_i(d)$.
  The first translation, however, when compiled into STRIPS introduces additional actions as well
  \shortcite{nebel:expressiveness}.
  Finally, a reset effect $c(d) := 0$ is obtained by setting all atoms $b_0(d)$, $i\in[0,n]$, to false.
}

The actions $a$ that belong to $Q$ are split into two classes. Actions $a$ \emph{that do not decrement any variable},
keep their names in $P=T(Q)$ but replace their $Inc(X)$ effects by propositional effects $\GT{X}$, and add the
precondition $\neg in(X)$ that disables $a$ when $X$ is in the stack:

\begin{enumerate}[1.]
  \item[4.] \textbf{Actions $a$ in $Q$ that decrement no variable}, keep the same names in $T(Q)$, the same preconditions
    and same effects, except that the effects $Inc(X)$ are replaced by propositional effects $\GT{X}$,
    if any, and in such a case, the precondition $\neg in(X)$ is added.
\end{enumerate}

Actions $a$ from $Q$ \emph{that decrement variables} (and hence introduce non-determinism)
map into actions in $P$ of type $a(X,d)$ where $X$ is a variable decremented by $X$
that is in the stack at depth $d$:
more of the variables that are decremented are in the stack when the action is applied:

\Omit{
\begin{enumerate}[1.]
  \item[5.] \textbf{Actions $a(d)$ for actions $a$ in $Q$ that decrement variables $X_1,\ldots,X_k$, $k\ge 1$,
    none of which is in the stack} inherit propositional preconditions and effects of $a$, and for each variable
    $Y$ that is increased by $a$, they include the precondition $\neg in(Y)$ and effect $\GT{Y}$.
    The parameter $d\in[0,|V|]$ stands for the current stack depth. The action $a(d)$ also has:
    \begin{enumerate}[$a)$]
      \item extra preconditions $depth(d)$ and $\neg in(X_i)$, $i\in[1,k]$, and $c(d) < Max$,
      \item extra non-deterministic effects $\GT{X_i}\,|\,\EQ{X_i}$, $i\in[1,k]$, and
      \item extra effect $c(d):= c(d)+1$ to increase the counter for level $d$.
    \end{enumerate}
\end{enumerate}

Similarly, the actions $a$ from $Q$ that decrement a variable $X$ that is in the stack are captured
through $a(X,d)$ actions:
}

\begin{enumerate}[1.]
  \item[5.] \textbf{Actions $a(X,d)$ for $a$ in $Q$ that decrement a variable $X$ in the stack at level $d$
    (and possibly others)} inherit propositional preconditions and effects of $a$, and for each variable $Y$
    that is increased by $a$, they include the precondition $\neg in(Y)$ and the effect $\GT{Y}$.
    The parameter $d\in[1,|V|]$ stands for the current stack depth. The action $a(X,d)$ also has:
    \begin{enumerate}[$a)$]
      \item extra precondition $index(X,d)$ (i.e., $d$ is level at which $X$ appears in stack),
      \item extra non-deterministic effects $\GT{X_i}\,|\,\EQ{X_i}$ for each $Dec(X_i)$ effect, and
      \item extra effects $c(d'):=0$ for each $d'$ such that $d \leq d' \leq |V|$ to reset the counters for the levels above or equal to $d$.
    \end{enumerate}
\end{enumerate}

In words, actions $a$ from $Q$ that do not decrement any variable map into a single action of the form $a$ in $P=T(Q)$,
while actions $a$ from $Q$ that decrement  variables map into actions
%$a(d)$, applicable when   none of such variables is in the stack, which is at level $d$, and actions
$a(X,d)$ applicable only when a variable $X$ decremented by $a$ is in the
stack at level $d$. The actions of the form $a$ in $P=T(Q)$ are deterministic, and only  actions $a(X,d)$ can generate
cycles in a strong cyclic policy for $P$.
%The key difference between actions $a(d)$ and $a(X,d)$ is in the effects on the counters:  actions $a(d)$ increase  the counter $c(d)$, while actions  $a(X,d)$
%increase no counter and reset  the  counters $c(d')$ for all  $d' \geq  d$. This distinction is essential to ensure that all the strong cyclic policies for $P$ terminate.
The reduction $P=T(Q)$ can be summarized as follows:

\begin{definition}[Reduction QNP to FOND]
  \label{def:reduction:qnp->fond}
  Let $Q\tup{F,V,I,O,G}$ be a QNP.
  The FOND problem $P=T(Q)$ is $P=\tup{F',I',O',G'}$ with:
  \begin{enumerate}[1.]
    \item $F' = F \cup \{c_T\} \cup \{ depth(d), c(d) \} \cup \{ in(X) \} \cup \{ index(X,d') \}$,
    \item $I' = I \cup \{ depth(0), \EQ{c_T}, \EQ{c(0)} \}$,
    \item $G' = G$,
    %%\item $O'=  \{a : a \in O^+\} \cup \{ a(d), a(Y,d') : a \in O^- \} \cup \{ Push(X,d'-1), Pop(X,d') \} \cup \{ Move \}$
    \item $O'=  \{a : a \in O^+\} \cup \{ a(Y,d') : a \in O^- \} \cup \{ Push(X,d'-1), Pop(X,d') \} \cup \{ Move \}$
  \end{enumerate}
  where $X$ ranges over $V$, $d$ and $d'$ range over $[0,|V|]$ and $[1,|V|]$ respectively,
  and $O^-$ and $O^+$ stand for the sets of actions in $O$ that decrement and do not decrement
  a variable respectively, and the variable $Y$ in $a(Y,d')$ ranges among the variables
  decremented by the action $a$ in $Q$.
  Preconditions and effects of the actions in $O'$ are described above in the text.
\end{definition}

\section{Properties}

\noindent Clearly, the reduction $P=T(Q)$  can be computed  in polynomial time:

\begin{theorem}
  \label{thm:translation:poly}
  Let $Q=\tup{F,V,I,O,G}$ be a QNP.
  The reduction $P=T(Q)$ can be computed in time that is polynomial in the size of $Q$.
\end{theorem}
\begin{proof}
  Let $n=|F|+|V|$ be the number of variables, propositional or numerical, in $P=T(Q)$.
  $P$ has $1+n$ counters of capacity $1+2^n$, each one requiring $1+n$ bits: the counter
  $c(d)$ is encoded in binary with bits $c(d,i)$, $i\in[0,n]$.
  $P$ also has $n$ atoms of form $depth(d)$, $|V|=O(n)$ atoms of form $in(X)$, and
  $n|V|=O(n^2)$ atoms of form $index(X,d)$. Therefore, $P$ has $O(|F|+n^2)=O(n^2)$
  propositional variables.

  $P$ has $|V|^2=O(n^2)$ push actions.
  Since $Push(X,d)$ has precondition $c(d)<Max$ and effect $c(d):=c(d)+1$,
  it gets compiled into $n$ actions of the form $Push(X,d,i)$, $i\in[0,n-1]$,
  where precondition $c(d)<Max$ is expressed as $\{\neg c(d,i)\}\cup\{c(d,j):j\in[0,i-1]\}$,
  and effect $c(d):=c(d)+1$ is expressed as $\{c(d,i)\}\cup\{\neg c(d,j):j\in[0,i-1]\}$.
  The pop actions do not modify counters, so there are $O(n^2)$ of them.
  The $Move$ action increments the counter $c_T$ and then, like for $Push(X,d)$,
  it gets compiled into $n$ different actions.
  Actions $a$ in $Q$ that do not decrement variables are translated into
  actions $a$ in $P$. Actions $a$ that decrement a variable get translated
  into actions $a(X,d)$; there are $O(n^2)$ such actions $a(X,d)$ in $P$ for
  each such action $a$ in $Q$.

  In total, $P$ has $O(n^2)$ propositional variables and $O(|O|n^2 + n^3)$ actions,
  where the cubic term accounts for the $Push(X,d,i)$ actions.
  These numbers (polynomially) bound the size of $P$.
  It is clear that producing each action in $P$ is straightforward and can be done
  in polynomial time.
\end{proof}

The second direct property of the translation is that due to the use of the counters,
all strong cyclic  policies $\pi$ for $P$ must terminate:

\begin{theorem}
  \label{thm:terminateT}
  Let $Q$ be a QNP.
  Any strong cyclic policy $\pi$ for $P=T(Q)$ is $P$-terminating.
\end{theorem}
\begin{proof}
  Let $\pi$ be a strong cyclic policy for $P$ and let $\tau = \bar s_0, \ldots, [\bar s_i, \ldots, \bar s_m]^*$
  be an infinite $\pi$-trajectory. We need to show that there is some variable $X$ that is
  decreased in one of these states and increased in none.
  %i.e., $\pi(\bar s_j)$ is a $Dec(X)$ action for some $j\in[i,m]$, and for no $j\in[i,m]$, $\pi(\bar s_j)$ is an $Inc(X)$ action.
  Clearly, $\pi(\bar s)$ for some $\bar s$ in the recurrent set must be a non-deterministic
  action, and this means it is an action of form $a(X,d)$.
  %Moreover, one of these actions must be an $a(X,d)$, as $a(d)$ actions increase the $c(d)$
  %counter and hence cannot form a cycle by themselves (as $c(d)$ would grow without a bound).
  The actions $a(X,d)$ require $X$ to be in the stack and then resets all counters
  $c(d')$ for $d' \ge d$ back to $0$.

  Let us pick an action $a(X,d)$ in the recurrent states of $\tau$ to be one with smallest
  stack depth $d$, and let $\bar s$ be one of such states where $\pi(\bar s)=a(X,d)$.
  In the state $\bar s$, the variable $X$ is in the stack at level $d$; i.e., $index(X,d)$ is true.
  We show next that this same atom must be true in all the other recurrent states in $\tau$.
  Indeed, if there is a recurrent state where $index(X,d)$ is false, it means that there are
  recurrent states where $X$ is popped from the stack, and others where it is pushed back at
  level $d$, as $\bar s$ is a recurrent state where $index(X,d)$ holds.
  Yet, each occurrence of the action $Push(X,d-1)$ needed to make $index(X,d)$ true increases the counter
  $c(d-1)$ that no action $a(Y,d')$ can reset with $d'<d$, due our choice of the action $a(X,d)$ as
  one with minimum $d$. As a result, it has to be the case that $X$ is in the stack at level $d$ in
  all the recurrent states of $\tau$, and hence no action that increases $X$ is applied while in a
  recurrent state (since increments of $X$ are disabled when $X$ is in the stack).
  Then, since there is a recurrent state where $X$ is decremented, the infinite $\pi$-trajectory
  $\tau$ is terminating.
  Therefore, the policy $\pi$ is $P$-terminating.
\end{proof}

In order to prove soundness and completeness, we establish a correspondence between the strong cyclic
policies of $Q$ and the strong cyclic policies of $P=T(Q)$. The policies cannot be the same, however,
as the reduction $T$, unlike the direct translation $T_D$, adds extra variables and actions.
Indeed, $T$ preserves the atoms $p$ and $\EQ{X}$ from $Q$, the latter being propositional, but adds
boolean variables and actions that ensure that the policies over of $T(Q)$, unlike the policies over
$T_D(Q)$, terminate.

Let $Q_M$ be the QNP obtained from the FOND problem $P=T(Q)$ by 1)~adding the numerical variables $X$ from
$Q$, 2)~replacing the effects $\GT{X}$ by $Inc(X)$, and the non-deterministic effects $\GT{X}\,|\,\EQ{X}$
by $Dec(X)$, and 3)~interpreting the preconditions and goal of the form $\EQ{X}$ and $\GT{X}$ in terms of
such variables (i.e., non-propositionally).

\begin{theorem}
  If $\pi$ solves the FOND problem $P=T(Q)$, $\pi$ solves the QNP $Q_M$.
\end{theorem}
\begin{proof}
  $P$ is the direct translation of $Q_M$, i.e.\ $P=T_D(Q_M)$.
  $\pi$ is $P$-terminating by Theorem~\ref{thm:terminateT} and strong cyclic for $P$ as it solves it.
  Therefore, by Theorem~\ref{thm:td:main}, $\pi$ solves $Q_M$.
\end{proof}

The QNP $Q_M$ can be thought of as the composition of the original QNP $Q$ and with a
deterministic model that encodes the state of the stack and counters. From this perspective,
the policy $\pi$ that solves $Q_M$ stands for a \emph{controller} $\pi_M$ for $Q$  that
has an internal memory $M$ comprised of the atoms that encode the stack and counters: actions
like $Push(X,d)$, $Pop(X,d)$ and $Move$ only affect the internal memory $M$ of the controller
$\pi_M$, actions $a$ that do not decrement any variable in $Q$, only affect the state of $Q$,
while actions $a(X,d)$ affect both the state of $Q$ and the internal memory $M$.
Due to the correspondence between the application of policy $\pi$ to the QNP $Q_M$ and the
application of the controller with memory $\pi_M$ to the QNP $Q$, it is then direct that:

\begin{theorem}[Soundness]
  \label{thm:policy-memory}
  If policy $\pi$ solves the FOND problem $P=T(Q)$, the controller $\pi_M$ solves $Q$.
\end{theorem}
\begin{proof}
  Each execution of $\pi$ in $Q_M$ generates a trajectory over $M$ and one over $Q$.
  Since $\pi$ solves $Q_M$, the latter must be terminating and goal reaching, but then
  they must be terminating and goal reaching in $Q$ that shares the same goal as $Q_M$
  and the same $Dec(X)$ and $Inc(X)$ actions.
\end{proof}

The inverse direction of this theorem is also true, but it does not give
us a completeness result. For that, we  need to show  that a policy $\pi$
that solves $Q$ determines  a policy $\pi'$ that solves $P=T(Q)$.
%We address this task now.

\subsection{Completeness}

We now assume that there is a policy $\pi$ that solves  $Q$ and want
to show that  there is a policy $\bar\pi$ that solves  $P=T(Q)$.
Since $\pi$ solves $Q$, $\pi$ is $Q$-terminating and also $P'$-terminating
where $P'=T_D(Q)$ is the direct translation of $Q$ (cf.\ Theorem~\ref{thm:td:main}).

Let $\G$ be the policy graph associated with $\pi$ in $P'$. By Theorem~\ref{thm:sieve},
\sieve reduces $\G$ to an acyclic graph. For the rest of this section, we assume that
\emph{\sieve is run until all edges that are associated with actions that decrement variables are eliminated}
rather than stopping as soon as the graph becomes acyclic.\footnote{Clearly, the modification
  on the stopping condition for \sieve does not affect its correctness since an acyclic
  graph remains acyclic when one or more edges are removed, and, on the other hand, if the
  original \sieve cannot reduced a component, the modified algorithm is not
  able to reduce it either.}
As a result, since $\pi$ solves $Q$, the resulting acyclic graph has no edge associated with a decrement of a variable.
Each edge removed by \sieve can be identified with a variable, and edges are removed
in batches by \sieve, each such $batch(C)$ associated with a component $C$ and a variable $X$
chosen by \sieve; i.e., in a given iteration, \sieve chooses a component $C$ and a
variable $X$, and removes all edges $(\bar s,\bar s')$ from $C$ such that $\pi(\bar s)$
is a $Dec(X)$ action (cf.\ Figure~\ref{alg:sieve2}).

Let us index the top SCCs processed by \sieve in topological order (i.e., if $C_i$ reaches
$C_j$ for $j\neq i$, then $i<j$), and let $scc(\bar s)$ be the index of the (top) component
that includes $\bar s$ (i.e., the index of the component that includes $\bar s$ in the
graph $\G$).
\sieve decomposes each component $C$ into a collection of \emph{nested} SCCs
that result of recursively removing edges from $C$. For each state $\bar s$,
let $\C_{\bar s}=\{C_{\bar s}^j\}_{j\geq1}$ be the collection of nested SCCs
that contain $\bar s$; i.e.,
\begin{enumerate}[$\bullet$]
  \item $C_{\bar s}^1=C_{k_1}$ where $k_1=scc(\bar s)$ is the index of the (top) component that contains $\bar s$, and
  \item for $j\geq1$, $C_{\bar s}^{j+1}=C_{k_{j+1}}$ where $\bar s$ is in the component $C_{k_{j+1}}$ of the graph that results when all edges in
    $\bigcup\{ batch(C_{k_i}) : i\in[1,j]\}$ have been removed by \sieve.
\end{enumerate}

For each state $\bar s$, let $stack(\bar s)$ be the sequence of variables
chosen by \sieve for each component in $\C_{\bar s}$. Observe that such
sequence contains no repetitions since once \sieve chooses variable $X$
for a component $C$, the same $X$ cannot be chosen later for another
component $C'$ contained in $C$. Also, if the action $\pi(\bar s)$ is
a $Dec(X)$ action for variable $X$, then $stack(\bar s)$ contains $X$
by the assumption that \sieve is run until all edges associated with
decrements of variables are eliminated.

We compare the stack $\alpha$ and $stack(\bar s)$, and say that $\alpha=X_1 \cdots X_n$ is
a \emph{prefix} of $stack(\bar s)=Z_1 \cdots Z_m$ if the latter can be obtained from the
former by pushing variables only; i.e., if $n\leq m$ and $X_i = Z_i$ for $i\in[1,n]$.
A property of the \sieve algorithm that we exploit in the completeness proof is the following:

\begin{theorem}
  \label{thm:scc:stack}
  Let $Q$ be a QNP and let $P'=T_D(Q)$ be its direct translation.
  If $\pi$ solves $Q$ and $\bar\tau=\bar s_0,\ldots,[\bar s_i,\ldots,\bar s_m]^*$ is an infinite
  $\pi$-trajectory in $P'$, there is a variable $X$ and a recurrent state $\bar s$ such that
  $\pi(\bar s)$ is a $Dec(X)$ action, and $X$ is in $stack(\bar s')$ for every recurrent
  state $\bar s'$ in $\bar\tau$.
\end{theorem}
\begin{proof}
  If $\pi$ solves $Q$, $\pi$ must be $P'$-terminating. Thus, there must be a variable
  that is decreased by $\pi$ in some recurrent state, and increased by $\pi$ in no recurrent
  state. At the same time, \sieve is complete and must remove variables (edges) in the policy
  graph until it becomes acyclic (cf.\ Theorem~\ref{thm:sieve}).
  Initially, all the recurrent states in $\bar\tau$ are in the same component but at one
  point \sieve removes a variable $X$ and splits the set of recurrent states into different
  and smaller components.
  From its definition, this means that $X$ appears in $stack(\bar s)$ for each
  recurrent state $\bar s$ in $\bar\tau$.
  Moreover, since the removal of $X$ leaves these states into two or more components,
  $\pi(\bar s)$ for one such state must be a $Dec(X)$ action.
\end{proof}

We now define the policy $\pi^*$ for $P=T(Q)$ that is determined by the policy $\pi$ that solves $Q$.
%Once we prove that $\bar\pi$ solves $P$, we will have the completeness proof for the reduction.
In the definition, we use the two functions $scc(\bar s)$ and $stack(\bar s)$ defined above
in terms of the execution of \sieve on the policy graph for $\pi$ on $T_D(Q)$.
The states over $P$ are denoted by triplets $\tup{\bar s,c,\alpha}$ where $\bar s$ is the state
in $T_D(Q)$, $c$ stands for the state of the counters $c(d)$, $d\in[0,|V|]$, and $c_T$,
and $\alpha$ stands for the state of the stack (given by the atoms $depth(d)$, $in(X)$, $index(X,d)$).
The policy $\pi^*$ for $P$ is defined at triplet $\tup{\bar s,c,\alpha}$ by
\begin{equation}
\label{def:pi}
\begin{cases}
   Pop(X,d) & \text{if $c_T < scc(\bar s)$, $X$ is top variable in $\alpha$, and $d=|\alpha|$, else} \\
      Move  & \text{if $c_T < scc(\bar s)$ and empty stack, else } \\
   Pop(X,d) & \text{if $X$ is top variable in $\alpha$, $d=|\alpha|$, and $\alpha$ is not a prefix of $stack(\bar s)$, else} \\
  Push(X,d) & \text{if $\alpha X$ is a prefix of $stack(\bar s)$ and $d=|\alpha|$, else} \\
          a & \text{if $\pi(\bar s)=a$ decrements no variable, else} \\
     a(X,d) & \text{if $\pi(\bar s)=a$ decrements $X$ at depth $d$ in $\alpha$ but no other var at depth $d' < d$.} \\
\end{cases}
\end{equation}

A first observation is that the policy $\pi^*$ is defined on every triplet $\tup{\bar s,c,\alpha}$
such that $\pi(\bar s)$ is defined.
The policy $\pi^*$ for the FOND problem $P=T(Q)$ on a triplet $\tup{\bar s,c,\alpha}$
advances the $c_T$ counter until it becomes equal to the index $scc(\bar s)$ of the SCC
in $\G$ that contains the node $\bar s$.
It then performs pops and pushes until $\alpha$ becomes equal to $stack(\bar s)$, and
finally applies the action $a$ selected by the policy $\pi$ on $Q$ using the action names
$a$ or $a(X,d)$ according to whether $a$ decrements no variable or decrements a variable,
which must be in $stack(\bar s)$ as discussed above. In the latter case, the variable $X$
for the action $a(X,d)$ is the variable $X$ decremented by $a$ that is
\emph{deepest in the stack}, at depth $d$.
The completeness result can then be expressed as follows:

\begin{theorem}[Completeness]
  \label{thm:qnp->fond:completeness}
  %Let $Q$ be a QNP and let $\pi$ be a policy for $Q$.
  If $\pi$ solves the QNP  $Q$, then the policy $\pi^*$ defined by \eqref{def:pi} solves the FOND problem $P=T(Q)$.
\end{theorem}
\begin{proof}
  From Theorem~\ref{thm:td:main}, $\pi$ solves $Q$ iff $\pi$ solves and terminates in $P'=T_D(Q)$.
  We will show that if $\pi$ solves and terminates in $P'$, then $\pi^*$ must solve $P=T(Q)$.
  Therefore, by forward reasoning, given that $\pi$ solves $Q$, then $\pi$ solves and terminates
  in $P'$ from which we obtain that $\pi^*$ solves $P$.

  We need to show that the policy $\pi^*$ is executable in $P$, and more precisely that
  1)~$\pi^*$ cannot generate non-goal states $\tup{\bar s,c,\alpha}$ where $\pi^*$ is not
  defined or defined but non applicable, and
  2)~$\pi^*$ cannot get trapped in a loop that only involves the extra actions $Move$, $Pop(X,d)$ and $Push(X,d)$.
  These two properties ensure that in any $\pi^*$-trajectory over $P$, if a non-goal state
  $\tup{\bar s,c,\alpha}$ is reached, an action $a$ or $a(X,d)$ will be the one
  changing the component $\bar s$ of the state when $\pi(\bar s)=a$, and that this will happen
  after a bounded number of applications of the extra actions $Move$, $Pop(X,d)$ and $Push(X,d)$
  that do not change $\bar s$.
  Since the effect of the actions $a$ or $a(X,d)$ on $\bar s$ in $P$ is the same as
  the effect of $a$ on $s$ in $P'$, it follows that $\pi^*$ will be strong cyclic for $P$
  if $\pi$ is strong cyclic for $P'$.
  Alternatively, 1) and 2) ensure that if $\pi^*$ is executable in $P$, it generate
  trajectories over the $\bar s$ components that are the same as those obtained by the
  policy $\pi$ over $P'$ except for a bounded number of steps where the $\bar s$ component
  in the states $\tup{\bar s,c,\alpha}$ does not change.

  \medskip
  Point 2) is direct.
  $Move$ increases the counter $c_T$ that no other action decreases.
  Pushes and pops are applied in order, either to flush out the stack when
  $c_T < scc(\bar s)$, or to make $\alpha=stack(\bar s)$.
  In the latter case, $\alpha$ is popped until it becomes a prefix of $stack(\bar s)$
  (flushed out in the extreme case), and then pushes take place to make $\alpha$
  equal to $stack(\bar s)$.  Hence, no loops that only involve $Move$, $Pop(X,d)$
  and $Push(X,d)$ actions are possible.

  \medskip
  Point 1) is more subtle.
  The policy $\pi^*$ is defined on all triplets $\tup{\bar s,c,\alpha}$ for which
  $\bar s$ is reachable by $\pi$. We first argue, that except for the preconditions
  on the counters $c(d)$ and $c_T$, the rest of the preconditions are true for the
  actions selected by $\pi^*$.
  Observe that every triplet $\tup{\bar s,c,\alpha}$ reached by $\pi^*$ is such that
  $\bar s$ is reachable by $\pi$, and thus $\pi(\bar s)$ is defined and applicable in
  $\bar s$.
  Second, for an action selected by $\pi^*$, its easy to see, except for $\neg in(Y)$
  when the action is $a$ or $a(X,d)$ and it increments $Y$, that its preconditions hold.
  To see that $\neg in(Y)$ also holds, observe that if the actions increments $Y$,
  then $stack(\bar s)$ cannot contain $Y$; if so, the collection $\C_{\bar s}$ of nested
  components for $\bar s$ has a component $C$ that contains a state where $Y$ is decremented
  while being incremented in $\bar s$, thus making $Y$ ineligible by \sieve.

  The actions that have preconditions on counters are of type $Push(\Box,d)$ with
  precondition $c(d)<Max$, and $Move$ with precondition $c_T<Max$.
  Here, $\Box$ is a placeholder that denotes any variable $X$ in $V$.
  For the top counter, $c_T < Max$ always hold since $c_T$ starts at 0, it is
  only increased to make it equal to $scc(\bar s)$, and the number of components in
  $\G$ is less than or equal the number of subsets of states which is less than $Max$.
  We are thus left to show $c(d) < Max$ by considering the only type of actions
  that increase $c(d)$: $Push(\Box,d)$.

  For this, we show that $\pi^*$ cannot generate a trajectory $\tilde\tau$ in $P$
  that contains a fragment $\tilde\tau'$ with $1 + 2^n$ (i.e.\ $Max$) actions of the
  form $Push(\Box,d)$ while no action of the form $a(\Box,d')$, $d' \leq d$,
  or $Push(\Box,d-1)$ as this would be the only way in which $c(d)$ may grow up to
  $1 + 2^n$: actions of the form $Push(\Box,d)$ increase $c(d)$ by 1, and the only
  actions that decreases $c(d)$, back to $0$, have the form $a(\Box,d')$ for $d'\leq d$,
  or $Push(\Box,d-1)$.

  Indeed, let $\tilde\tau'=\tup{\bar s_1,c_1,\alpha_1}, \tup{\bar s_2,c_2,\alpha_2}, \ldots$
  be such a fragment, and let $1=i_1<i_2<\cdots<i_m$, for $m=1+2^n$, be the indices for
  the triplets in $\tilde\tau'$ on which the policy $\pi^*$ selects an action of type
  $Push(\Box,d)$.
  Observe that between each pair of such indices, there must be one triplet where an action
  of type $a$ or $a(\Box,d')$ is applied: two pushes at the same stack depth must be mediated
  by at least one such action.

  Let $i^*_1<i^*_2<\cdots$ be the indices such that $i^*_k$ is the first index after $i_k$
  where the action selected by $\pi^*$ is of type $a$ or $a(\Box,d')$, $k\in[1,m]$.
  Since the total number of states is less than or equal to $m$, there is some $\bar s$
  that repeats. Without loss of generality, let us assume that $\bar s_{i^*_1}=\bar s_{i^*_m}$.
  The policy $\pi$ loops in $P'$ on a set $\mathcal{R}$ of recurrent
  states that includes $\{\bar s_{i^*_k} : k\in[1,m] \}$.
  By Theorem~\ref{thm:scc:stack}, there is a variable $X$ that is decremented by $\pi$
  while looping in $\mathcal{R}$ such that $X$ belongs to each $stack(\bar s_{i^*_k})$, $k\in[1,m]$.
  We choose $X$ to be such variable appearing deepest in the stacks.
  Therefore, there is index $k\geq 1$ such that $\pi^*(\tup{\bar s_k,c_k,\alpha_k})=a(X,d')$
  where $d'$ is the depth of $X$ in $\alpha_k$.
  Since $X$ also belongs to $\alpha_1$, it must be the case $d'\leq |\alpha_1|=d$,
  the latter inequality since $\pi^*(\tup{\bar s_{i_1},c_{i_1},\alpha_{i_1}})$ is of
  type $Push(\Box,d)$.
  This is a contradiction with the assumption that $\tilde\tau'$ contains no action
  of type $a(\Box,d')$ for $d'\leq d$.
\end{proof}

% Discussion about FSM vs Flat policies
The second reduction from QNPs into FOND problems may be used to compute policies
for a given QNP from policies of the resulting FOND. The reduction is a sound and
complete mapping.
As mentioned above, the resulting QNP policies correspond to
controllers that map states $\bar s$ and controller states, pairs $\tup{c,\alpha}$
that encode the state of the (bounded) counters and stack, into actions.
Hence, there is still the question of whether a QNP solvable by such a controller
is solvable by a flat policy (as given in Definition~\ref{def:qnp:solution}).
In the examples below,  the controllers  found by the FOND planner
over the translation yield  flat policies where the selection of actions
does not depend on the internal controller state, but more generally  the
question remains open and beyond the scope of this paper.

\Omit{%
We have spent time trying to answer it without success.
Indeed, there are arguments that support both an affirmative and a negative answer.
For the former, QNPs contain no hidden information and thus it
seems that they satisfy the Markov property: the current state decouples the
past (history) from the future.
For the latter, it seems that in order to select an appropriate action,
an agent must keep in memory what variables it is currently trying to
get to zero to avoid any increment of them (such information is in the stack).
On the other hand, the other general approach for solving QNPs, based on LTL synthesis
and discussed below, is also incapable of resolving this issue as well as
the inefficient algorithm of brute-force enumeration and termination test
of all strong-cyclic policies.
}

\section{Examples}

Let us illustrate the translation and  its solution with a simple example and two variations.
The base QNP is $Q_1=\tup{F,V,I,O,G}$ where there are two boolean variables $F=\{p,g\}$, two
numerical variables $V=\{n,m\}$, the initial and goal states are $I=\{p,\GT{n},\GT{m}\}$ and
$G=\{g\}$ respectively, and the  four  actions in $O$ are $a_1 = \abst{p, \GT{n}}{\neg p, \DEC{n}}$,
$a_2 = \abst{\neg p}{p}$, and $\textit{fin}_1 = \abst{\EQ{n}}{g}$ and $\textit{fin}_2 = \abst{\EQ{m}}{g}$,
where the last two actions are used to  capture the \emph{disjunctive goal} $\EQ{n} \lor \EQ{m}$.
The QNP $Q_2$ is like $Q_1$ except that $a_2 = \abst{\neg p}{p, \INC{n}}$, while the QNP $Q_3$ is
like $Q_1$ (and $Q_2$) except that $a_2 = \abst{\neg p, \GT{m}}{p, \INC{n}, \DEC{m}}$.

Figure \ref{fig:syn} shows the solutions of the FOND problems $P_1=T(Q_1)$ and $P_3=T(Q_3)$
for $Q_1$ and $Q_3$ that are obtained  with the \qnptofond translator (see below), and a FOND planner.
The QNP $P_2=T(Q_2)$, on the other hand, has no solution.

In order to understand these results, recall that in the   FOND problem $T(Q)$  for any QNP $Q$,
actions that decrement variables may be applied  only when some of the decremented variables are  in the stack,
and actions that increment variables may be applied only when  none of the incremented variables is in the stack.
In addition, the translation prevents  loops  where a variable is pushed and popped from the stack,
except when  the stack contains another variable  throughout the loop that is decremented,
as otherwise the counters would grow without bounds. With these elements in mind, we can turn to the solutions
of $Q_1$ and $Q_3$, and the lack of solutions for $Q_2$.

The solution for $Q_1$ is direct. The idea is to use action $a_1$ to decrement $n$ down
to zero, and then apply $\textit{fin}_1$ to reach the goal. However, $a_1$ is applicable only
if $n$ is in the stack and $p$ is true. Then, $n$ must be pushed into the stack before the first
application of $a_1$, and $p$ must be set to true using $a_2$ after each use of $a_1$ as
observed in Fig.~\ref{fig:syn}(a).

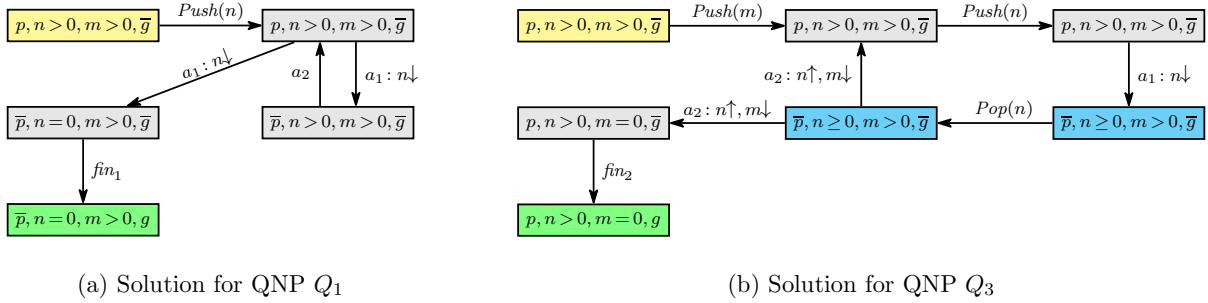
\begin{figure}[t]
  \centering
  \resizebox{\textwidth}{!}{
  \begin{tabular}{ccc}
    \begin{tikzpicture}[thick,>={Stealth[inset=2pt,length=8pt,angle'=33,round]},font={\scriptsize},node distance=1cm,qs/.style={draw=black,fill=gray!20!white},init/.style={qs,fill=yellow!50!white},goal/.style={qs,fill=green!50!white}]%, show background rectangle, show background grid]
      \node[init]               (n0) { $p, \GT{n}, \GT{m}, \overline{g}$ };
      \node[qs, right =1.6cm of n0]  (n1) { $p, \GT{n}, \GT{m}, \overline{g}$ };
      \node[qs, below = of n1]  (n2) { $\overline{p}, \GT{n}, \GT{m}, \overline{g}$ };
      \node[qs, left =1.6cm of n2]  (n3) { $\overline{p}, \EQ{n}, \GT{m}, \overline{g}$ };
      \node[goal,below = of n3] (ng) { $\overline{p}, \EQ{n}, \GT{m}, g$ };
      \path[->] (n0) edge[] node[above,yshift=-2] { $\textit{Push}(n)$ } (n1);
      \path[->] (n1) edge[transform canvas={xshift=8}] node[right,yshift=0] { $a_1\!:\DEC{n}$ } (n2);
      \path[->] (n1) edge[] node[sloped,yshift=5] { $a_1\!:\DEC{n}$ } (n3);
      \path[->] (n2) edge[transform canvas={xshift=-8}] node[left,yshift=0] { $a_2$ } (n1);
      \path[->] (n3) edge[] node[right,yshift=0] { $\textit{fin}_1$ } (ng);
    \end{tikzpicture} & \qquad\qquad &
    \begin{tikzpicture}[thick,>={Stealth[inset=2pt,length=8pt,angle'=33,round]},font={\scriptsize},node distance=1cm,qs/.style={draw=black,fill=gray!20!white},init/.style={qs,fill=yellow!50!white},goal/.style={qs,fill=green!50!white},qa/.style={qs,fill=cyan!50!white}]%, show background rectangle, show background grid]
      \node[init, right=5cm of n0]   (m0) { $p, \GT{n}, \GT{m}, \overline{g}$ };
      \node[qs, right =1.8cm of m0]  (m1) { $p, \GT{n}, \GT{m}, \overline{g}$ };
      \node[qs, right =1.8cm of m1]  (m2) { $p, \GT{n}, \GT{m}, \overline{g}$ };
      \node[qa, below= of m2]  (m3) { $\overline{p}, \GE{n}, \GT{m}, \overline{g}$ };
      \node[qa, left =1.8cm of m3]  (m4) { $\overline{p}, \GE{n}, \GT{m}, \overline{g}$ };
      \node[qs, left =1.8cm of m4]  (m5) { $p, \GT{n}, \EQ{m}, \overline{g}$ };
      \node[goal, below = of m5]  (mg) { $p, \GT{n}, \EQ{m}, g$ };
      \path[->] (m0) edge[] node[above,yshift=-2] { $\textit{Push}(m)$ } (m1);
      \path[->] (m1) edge[] node[above,yshift=-2] { $\textit{Push}(n)$ } (m2);
      \path[->] (m2) edge[] node[right,yshift=0] { $a_1\!:\DEC{n}$ } (m3);
      \path[->] (m3) edge[] node[above,yshift=-2,xshift=4] { $\textit{Pop}(n)$ } (m4);
      \path[->] (m4) edge[] node[left,yshift=0,xshift=0] { $a_2\!:\INC{n},\DEC{m}$ } (m1);
      \path[->] (m4) edge[] node[above,yshift=-2,xshift=0] { $a_2\!:\INC{n},\DEC{m}$ } (m5);
      \path[->] (m5) edge[] node[right,yshift=0,xshift=0] { $\textit{fin}_2$ } (mg);
    \end{tikzpicture} \\[1em]
    (a) Solution for QNP $Q_1$ & & (b) Solution for QNP $Q_3$
  \end{tabular}
  }
  \caption{Solutions to  the FOND translations $T(Q_1)$ and $T(Q_3)$  of the QNPs  $Q_1$ and $Q_3$ in the text.
    Nodes represent states in the translations  (i.e., boolean QNP states augmented with stack and counters) but
    only the QNP part is shown.     Edges correspond to actions from Q or actions that manipulate the stack and counters.
    Edges are annotated with action labels and their effect on the numerical variables.
    Blue nodes represent multiple  QNP states; e.g., the node $\{\overline{p}, \GE{n},\GT{m},\overline{g}\}$ for $Q_3$ represents the QNP states where
    $\{\overline{p},\GT{m},\overline{g}\}$ and there is no restriction on the value of $n$.
    In both controllers, the initial state is the top leftmost state (in yellow) and the goal is the rightmost state at the bottom (green).
 %    The policies are strong cyclic and terminating.
    The solution for $Q_1$ decrements $n$ with action $a_1$ until it becomes zero, and  $a_1$ requires $p$ and thus $a_1$ is interleaved with $a_2$ that makes $p$ true.
    In $Q_3$, the action $a_2$ is changed to increment $n$ as well, and to  decrement another variable $m$.
    The solution for $Q_3$ found by the solver reaches the goal by decreasing $m$ to zero using action $a_2$, while using $a_1$
    to restore the preconditions of $a_2$, and ignoring the variable $n$. Another solution could be obtained by applying the
    action $\textit{fin}_1$ to the states where $n=0$ but it would involve more controller states (not shown).
%         decrements $m$, instead of $n$, using $a_2$ that requires $\neg p$ but it also increments $n$. The action $a_1$ makes $p$ false but
    %     decrements $n$. Hence, executions of $a_1$ and $a_2$ are interleaved but each execution
   % of $a_1$ requires $n$ to be in the stack while each execution of $a_2$ requires $n$ to be out of the stack.
    %Hence, $n$ is pushed and popped between executions of $a_2$.
  }
  \label{fig:syn}
\end{figure}

The QNP $Q_2$ has no solution because the action $a_2$ that must be used to restore the
precondition $p$ of $a_1$ in the loop increments the variable $n$ as well. Since action $a_1$ requires
$n$ to be in the stack, which prevents an action like $a_2$ in $Q_2$ to execute, the only possibility
is to pop $n$ from the stack before executing $a_2$, but then variable $n$ should be
pushed and popped from the stack in a loop  without the presence of a another variable in
the stack that is decremented; a condition that is precluded by the translation, and which
is necessary for the loop to terminate.

It is precisely the presence of such extra variable $m$ that is decremented by $a_2$ and not
incremented by any action that  makes the  QNP $Q_3$  solvable. While the goal can be
achieved by either reaching the condition $n=0$ or $m=0$ and then applying the action $\textit{fin}_1$
or $\textit{fin}_2$ respectively, the solution found by the FOND planner over $T(Q_3)$
focuses on  decrementing $m$ down to zero using $a_2$ instead, and then using $\textit{fin}_2$ to reach the goal.
The loop involving an unbounded number of pushes/pops of $n$ in the stack is permitted in the solution
because the variable $m$ is in the stack and it is decremented in each iteration of the loop. In the translation this means
that decrements of $m$ reset the counters associated with the variables like $n$ that are above
$m$ in the stack, cf.\ Fig.~\ref{fig:syn}(b).

\Omit{
In order to do so, $m$ needs to be in the stack and $p$ must be false each time $a_2$ is applied.
The former is achieved with the initial push while the latter by interleaving the executions
of $a_2$ with executions of $a_1$, yet $a_1$ requires $n$ to be in the stack and thus the plan
keeps pushing $n$ into and popping it from the stack.

The solution for $Q_3$ sheds light on the reasons for the unsolvability of $Q_2$.
It is because $n$ would need to enter and leave the stack an unbounded number of times
but there is no other variable to decrement that can be used to permit such unbounded
operations on the stack.}

Finally,  an extra  dummy variable like $z$ together with a dummy action $a_3=\abst{\GT{z}}{\DEC{z}}$
would not make the QNP $Q_2$ solvable either, because while such an action could be used to
render a terminating loop involving actions $a_1$ and $a_2$, along with $a_3$, it would move
the unsolvability to the subproblem that results when the variable $z$ becomes zero.
The policy would need to map states where $z=0$ holds into goal states,
and then, the  same obstacles arise.

\section{Extensions,  Variations, and Memoryless Controllers}

For simplicity, QNPs have been defined with certain syntactic restrictions that do not
limit their expressive power. In particular, there are no actions with non-deterministic
effects on boolean variables, and there are no effects that can map a literal $\EQ{X}$
non-deterministically into the literals $\GT{X}$ and $\EQ{X}$, as decrements require
$\GT{X}$ as a precondition, and increments yield the outcome $\GT{X}$.
Yet, these two restrictions are not essential and can be bypassed.
As we have seen before, non-deterministic effects on boolean variables as found in
strong and strong cyclic FOND planning can be obtained in QNPs by using additional
numerical variables (Section~6).
%%In particular, fair non-deterministic effects $E_1 | E_2$
%%are captured by means of the (cyclic) sequence of   effects $Inc(X)$, $Dec(Y)$, ``if $Y=0$ then Goal'',
%% $Dec(X)$, ``If $X=0$ then $E_1$ and break'', $Dec(X)$,  ``if $X=0$ then $E_2$ and break'', and back to
%% first $Dec(X)$, while unfair (adversarial) non-deterministic effects $E_1 | E_2$
% are captured by the same sequence but over the variable $X$ only (see Section~6).
%% Moreover, in this manner,  QNPs can deal with a combination of fair and unfair actions, in what has been called \emph{dual FOND planning}
%% where solutions are strong cyclic policies that result in cycles that must be fair
%% (some outcome of a fair action is skipped forever) \cite{geffner:icaps2018}.
%%% Nope: not so clear ..
%%
%% dual fond: if all fair paths reach goal; fair if no effect of fair action skipped forever
%% analogous: only cycles .. must be unfair; some fair effect skipped forever ..
%%
%
Likewise, a sequence of two consecutive effects $Inc(X)$ and $Dec(X)$ can be used to
emulate the effect of an action that leaves the value of $X$ completely uncertain; i.e.,
it may increase $X$, decrease $X$, or leave the value of $X$ unchanged.
% Indeed such uncertain effects are sometimes useful and we will refer to them
% a \textbf{reset} effects, and denote them as $Reset(X)$.

There are also syntactic conditions that when satisfied, make QNPs simpler. For example,
if the numerical variables in a QNP $Q$ can be linearly ordered as $X_1,X_2,\ldots$ so that the actions
that increase a variable $X_i$ also decrease a variable $X_j$ that appears later in the
ordering, then every policy $\pi$ that solves the direct translation $P=T_D(Q)$ of $Q$
will necessarily solve $Q$ as well, as any such policy will terminate in $Q$.
Indeed, if $X_\ell$ is the last variable in the ordering that is decreased in a cycle
induced by $\pi$, the cycle cannot include a different state where the variable $X_\ell$
is increased, as otherwise, the condition implies that another variable $X_j$ appearing
later in the ordering must be decreased in the cycle. For such well-ordered QNPs,
the simpler, direct translation $T_D$ is thus both sound and complete.

Finally, recall that a policy $\pi$ that solves the FOND $P=T(Q)$ obtained from the
translation of a QNP $Q$, defines a policy that can be understood as
a memory-extended controller  that solves $Q$ using extra boolean variables
and actions. Often, however, the policy $\pi(\bar s,m)$ obtained from $P$, where $s$ is
the state over $Q$ and $m$ is the memory state, can be projected onto a
\emph{memoryless controller}  $\pi'$ for $Q$ which does not use the extra variables
or actions. The projection is possible and immediate  when there is no state
$s$ in the controller where the actions $\pi(\bar s,m)$ and $\pi(\bar s,m')$
selected by the policy over two different memory states are associated with different actions in $Q$.
In such a case, all states $s$ can be  associated with a single action $a$ from
$Q$ (there must be one such action as otherwise $\pi$ would not map $\bar s$ into a goal state),
and the  memoryless policy $\pi'$ for $Q$ is then  simply  $\pi'(\bar s)=a$.

\section{QNPs and Generalized Planning}

QNPs were introduced by \citeay{sid:aaai2011} as a useful model for planning with loops and for generalized planning
% \cite{levesque:loops,srivastava08learning,bonet09automatic,srivastava:generalized,hu10correctness,hu:generalized,bonet:ijcai2015,BelleL16}.%% from ijcai2017
\shortcite{levesque:loops,bonet:icaps2009,srivastava:generalized,hu:levesque,bonet:ijcai2015}. %% from ijcai2017
In the basic formulation \shortcite{hu:generalized}, a generalized planning problem
is a  collection $\Q$ of planning instances $P$ that share the same set of
(ground) actions and the same set state features. The solution of the generalized problem $\Q$
is then  a mapping from feature valuations into actions that solves each of the instances in $\Q$.
This basic formulation was then extended to domains where the ground actions
change from instance to instance, as in most relational domains, like Blocksworld,
where  the actions are determined by a small number of action schemas and
object names. This formulation is achieved by means of QNPs \shortcite{bonet:ijcai2018}
where a single QNP is shown to be capable of representing a suitable abstraction
of the concrete instances $P$ involving different ground actions.

A  QNP is a \emph{sound abstraction} of a family $\Q$ of concrete problems $P$ from a common domain
when the boolean and numerical variables $p$ and $n$ in the QNP accurately
represent and track the value changes of certain boolean and numerical state features $\phi_p$ and $\phi_n$
in each of the instances. More precisely, a QNP action $\bar{a}=\abst{Pre}{\Eff}$ is sound relative to $\Q$
if in any (reachable) state $s$ over an instance $P$ in $\Q$, if the formula $Pre$ is true in $s$,  with the
QNP variables $p$ and $n$ replaced by the state feature functions $\phi_p$ and $\phi_n$, then there is
an action in $P$ that induces a state transition $(s,s')$ that agrees with the effects of the abstract
action $a$, once again,  with the QNP variables replaced by the corresponding state features.

For example, in Blocksworld, a QNP action $\bar{a}=\abst{\neg H,\GT{n(x)}}{H,\DEC{n(x)}}$ provides  a suitable abstraction of
the action of picking up a block from above a designated block $x$. In this abstraction, the variable $H$ is associated
with the boolean state feature $\phi_H$ that captures when the arm is empty, and the variable $n(x)$ is
associated with the numerical state feature $\phi_n$ that captures the number of blocks above $x$.
The abstract action is sound in the sense that for any state $s$ of a Blocksworld instance $P$,
if $\phi_H(s)$ is false and $\phi_{n(x)}(s) > 0$, there is an action $b$ in $P$ that induces a state
transition $(s,s')$ such that $\phi_H(s')$ is true and $\phi_{n(x)}(s') < \phi_{n(x)}(s)$, in agreement with $\bar{a}$.
The concrete action $b$ is then said to instantiate the abstract action $\bar{a}$ in the state $s$ of the instance.

In general, if all the QNP actions are sound relative to $\Q$ and suitable conditions apply to the
initial and goal conditions of the QNP in relation to $\Q$,  any policy $\pi$ that solves the QNP
provides  a solution to  $\Q$; i.e., the policy $\pi$ can be applied to any instance $P$ in $\Q$
by interpreting the variables in the QNP in terms of the state features \shortcite{bonet:ijcai2018}.
%It is not necessary, however, for the QNP actions to be sound relative to $\Q$ for this to happen:
%indeed, if the policy $\pi$ applies an abstract (QNP) action $\bar{a}$ on some states of the instances,
%it is indeed sufficient for $\bar{a}$ to be sound in those states only.
More recently, it has been shown how these   QNPs can be learned directly
from a  PDDL  description of the domain and a number of sampled  instances and their plans \shortcite{bonet:aaai2019},
and also how to obtain testable logical conditions to check the soundness of a QNP-based abstraction
for an instance $P$ of a PDDL domain description \shortcite{bonet:ijcai2019}.
The QNPs used in the experiments below are variations of QNPs learned from samples.

A final question about  QNPs for generalized planning is what are the generalized planning problems
for which QNPs provide a suitable abstraction and solution method. It turns out that with no
restrictions on the state features $\phi_p$ and $\phi_n$ that can be abstracted into the QNP,  % variables $p$ and $n$,
there is indeed, no limit. The solution to any family $\Q$ of planning problems
can be expressed compactly  in terms of a single QNP   action  $\bar{a}=\abst{\GT{V^*}}{\DEC{V^*}}$
that involves a single numerical variable  $V^*$  associated with the  feature $\phi_{V^*}$ that measures
the optimal cost of reaching the goal from a given state. The QNP action is sound  and just says to move
in the direction of the goal. However, the application of this abstract action  in a concrete state requires the computation of
optimal costs for each successor state, which is in general intractable in the number of problem variables. Thus, a reasonable restriction
is that the QNP variables should represent ``reasonable'' features, and in particular, features
that can be computed in polynomial (perhaps linear) time.

\section{Implementation: \Qnptofond}

The reduction from QNPs to FOND problems has been implemented and it can be found in the
Github repository \url{https://github.com/bonetblai/qnp2fond}. The reduction produces a
FOND problem without conditional effects which is desirable since some FOND planners
do not support them. The reduction may be parametrized in terms of
the maximum range of  counters and the maximum stack depth, but their  default values are
those  used in the proofs that ensure completeness in general.
In some cases, however, the reduction can be made simpler without compromising soundness or completeness.
\Omit{
\textcolor{red}{\bf REVISE THIS PARAGRAPH: direct translation need to be amended with counter
resets for the resulting translation to be sound and complete; that is, when optimizing translation,
if the action decrements $X$ that no other action incs, it is not sufficient to do the
direct translation of the action. Instead, the direct translation need to be augmented
with effects that reset all stack counters. The idea behind the optimization is that such
variable can be initially pushed into the stack and any action that decrements it, has
as side effect to reset all stack counters..}
}
For example, if no numerical variable in the QNP is incremented by an action,
the more compact and efficient  direct translation $T_D$ is sound and complete.
The same holds if the QNP is well-ordered as defined above, where
variables may be ordered as $X_1,\ldots,X_n$ such that actions that increment
a variable $X_i$ decrement a variable $X_j$ for  $i < j$.

On the other hand, if there are actions that increase variables but there are  no actions
that increase a variable $X$, there is no need to add push and pop actions for $X$, nor preconditions for $X$ to
be in the stack when decrementing it.
This is possible since $X$ may be assumed to be always at the bottom of the stack (as if an implicit
action that pushes $X$ has been executed before any other action)
and because there is no need to pop $X$ as no action increments it. However, according to the
translation, every action that decreases $X$ must reset all stack counters.
Interestingly, this and the other simplifications are special cases of a more general simplification
that can be applied when the QNP contains a subset $S$ of well-ordered numerical variables,
as defined above;  indeed, a subset of variables $X$ that are not incremented by any action is such a subset.
For such a subset $S$, one can assume that all variables in $S$ are in the stack and thus
simplify the translation by $(a)$ not generating push/pop actions for the variables in $S$,
$(b)$ not adding extra preconditions to actions that decrement a variable in $S$ (such actions already
fulfill the requirement that one of its decremented variables must be in the stack), and
$(c)$ requiring that actions that decrement any variable in $S$, reset all the stack counters.\footnote{Observe
  that the simplification cannot affect completeness as no action becomes inapplicable since
  preconditions are removed and counters are potentially reset more often.
  For soundness, suppose that there is a non-terminating execution and let $\R$ be the set of
  recurrent actions in such an execution.
  It cannot be the case that $\R$  contains only actions that affect numerical variables outside $S$
  since such actions are not affected by the simplification and thus are subject to the full translation.
  Hence, there is at least one action in $\R$ that either increments or decrements a variable in $S$.
  If the action increments a variable in $S$, by definition, the action also decrements another variable
  in $S$, one that comes later in the ordering. By inductive reasoning, $\R$ contains an action that
  decrements a ``last'' variable $X_\ell$ in $\R$ that is not incremented by any other action in $\R$.
  Hence, such a variable eventually becomes zero and the action cannot be applied afterwards, contradicting
  the choice of $\R$.}
When $S$ contains all the numerical variables in the QNP, then there is no need to have a stack
and the simplified translation $T$ simply becomes the direct translation $T_D$.

\Qnptofond supports the optimization for variables that are not
incremented by any action. The general optimization involving a subset of well ordered variables is
not yet implemented since finding such a subset and ordering is intractable.\footnote{A subset $S$ and
  ordering can be found by solving a simple SAT theory, while a maximum-size subset and ordering can be
  found with a weighted-max SAT solver, or by doing multiple calls to a SAT solver.}
There also options for disabling the optimization, and even to  force the direct translation
whose  solutions need to be checked with \sieve. By default, the options are to  use the optimization
whenever possible, and to use the maximum range for the counters and stack depth that ensure completeness
in general.

\section{Experiments}

We illustrate  the performance of  the QNP  translator and solver
over some QNPs that capture abstraction of generalized planning problems.
There is no useful baseline for evaluating the use of the translator in combination with
FOND planners. The only other complete QNP planner would result from translating QNP
problems into LTL synthesis tasks but the comparison would be unfair because % , as mentioned earlier, the
LTL synthesis is computationally harder than QNP planning.
There is also no complete generate-and-test QNP planner reported, which would have to
generate the strong cyclic policies of the direct FOND translation, one by one, while
checking them for termination.

In the examples, the resulting FOND problems $T(Q)$  are solved with FOND-SAT \shortcite{geffner:fond-sat},
a general SAT-based FOND planner that is available at \url{https://github.com/tomsons22/FOND-SAT}.
This planner calls a SAT solver multiple times.\footnote{In each call to the SAT solver, FOND-SAT
  tries to find a solution (controller) with a given number of states (budget). If no controller
  is found, the budget is increased by 1, and repeat until one is found.}
%The SAT solver used by FOND-SAT is Minisat \cite{minisat}.
The SAT solver used is Glucose 4.1 \shortcite{glucose} which builds on Minisat \shortcite{minisat}.
FOND-SAT solves a FOND problem by constructing a compact controller where each node
represents one or more states. When depicting the  solutions found, controller nodes that
represent more than one state are shown in blue,  while the nodes  that correspond to the
initial  and goal QNP states are  shown in yellow and green respectively.
% The nodes of the controller are associated with the boolean valuation of the QNP features
% and no information about the stack or counters is shown.

\subsection{Clearing a Block}

A general plan for clearing a given block $x$ in a Blocksworld instance
can be obtained by solving the following abstraction expressed as the QNP
problem $Q=\tup{F,V,I,O,G}$ where $F=\{H\}$ contains a
boolean variable $H$ that represents if a block is being held, $V=\{n\}$
contains a numerical variable that counts the number of blocks above $x$,
the initial situation $I=\{\neg H, \GT{n}\}$ assumes that block $x$ is
not clear and that gripper is empty, and the goal situation $G=\{\EQ{n}\}$
expresses that there are no blocks on top of $x$. There are four actions
in $O$:\footnote{From the point of view of generalized planning, the last
  QNP  action, $\textit{Pick-other} = \abst{\neg H}{H}$ is not sound in the
  Blocksworld domain because on the states where there is a single tower and
  the gripper is empty, the abstract action is applicable yet no concrete
  action corresponds to it.
  The obtained policy however is sound as it never prescribes the action \textit{Pick-other}.}
  %On the other hand, the abstraction can be made sound by adding more features or actions.
  %We opt for the simpler abstraction in order to improve clarity.}
  %there are states where all the blocks are beneath $x$,
  %and then $x$ is the only block that can be picked. This, however, does
  %not affect, the soundness of the policy, as indeed, such states are goal
  %states in this task.}
%
% from blocks_on.qnp
% Putaway
% 1 holding 1
% 1 holding 0
% Pick-above-x
% 2 n 1 holding 0
% 2 n 0 holding 1
% Put-above-x
% 1 holding 1
% 2 n 1 holding 0
% Pick-other
% 1 holding 0
% 1 holding 1
%
\begin{enumerate}[--]
  \item $\textit{Putaway}        = \abst{H}{\neg H}$ to put the block being held on the table or on a block not above $x$,
  \item $\textit{Pick-above-$x$} = \abst{\neg H, \GT{n}}{H, \DEC{n}}$ to pick the top block above $x$,
  \item $\textit{Put-above-$x$}  = \abst{H}{\neg H, \INC{n}}$ to put the block being held on the top  block above $x$, and
  \item $\textit{Pick-other} = \abst{\neg H}{H}$ to pick a block not above $x$.
\end{enumerate}

\Qnptofond translates the QNP $Q$ into the FOND problem $P=T(Q)$ that has 20 atoms and 16 actions in less than 0.01 seconds.
FOND-SAT solves $P$ in 0.08 seconds after 3 calls to the SAT solver that require 0.01 seconds in total.
The solution produced by FOND-SAT is the controller shown in Figure~\ref{fig:clear}
which depicts actions from the QNP as well as the added actions to manipulate the stack.
The resulting policy is a finite-state controller that can be
converted into the memoryless policy $\pi'$ for the QNP that picks a block above $x$
when the gripper is free, and puts away the block being held otherwise.

\begin{figure}[t]
  \centering
  \begin{tikzpicture}[thick,>={Stealth[inset=2pt,length=8pt,angle'=33,round]},font={\footnotesize},node distance=2cm,qs/.style={draw=black,fill=gray!20!white},init/.style={qs,fill=yellow!50!white},goal/.style={qs,fill=green!50!white}]%, show background rectangle, show background grid]
    \node[init] (n0) at  (0,0) { $\overline{H}, \GT{n}$ };
    \node[qs,right = of n0]   (n1) { $\overline{H}, \GT{n}$ };
    \node[qs,right = 3.5cm of n1]   (n2) { $H, \GT{n}$ };
    \node[goal,right = 1.5cm of n2]   (ng) { $H, \EQ{n}$ };
    \path[->] (n0) edge[] node[above,yshift=-2] { $\textit{Push}(n)$ } (n1);
    \path[->] (n1) edge[] node[above,yshift=0] { $\textit{Pick-above-x}:\DEC{n}$ } (n2);
    \path[->] (n1) edge[out=40,in=140,looseness=0.7] node[above,yshift=0] { $\textit{Pick-above-x}:\DEC{n}$ } (ng);
    \path[->] (n2) edge[out=220,in=320,looseness=0.7] node[below,yshift=0] { $\textit{Putaway}$ } (n1);
  \end{tikzpicture}
  \caption{%
    Solution of the FOND translation for the QNP $Q$ for clearing a block $x$.
    Nodes represent states in the translation (i.e., QNP states augmented with stack and counters) but only the QNP part is shown.
    Edges correspond to actions from $Q$ or actions that manipulate the stack and counters. Edges are annotated with action labels and their effect on the numerical variables. %(the effects on numerical variables are shown).
    The initial state is the leftmost state and the goal is the rightmost one. % (only one goal state shown; goal is $n=0$).
    The policy is strong cyclic and terminating.
  }
  \label{fig:clear}
\end{figure}
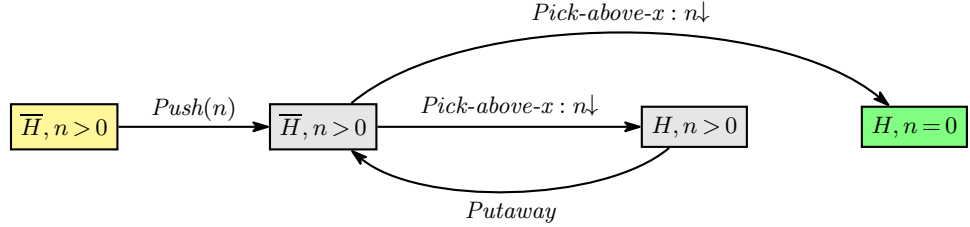

\subsection{Placing a Block on Top of Another}

A general plan for placing block $x$ on top of block $y$  may be
obtained by solving a suitable QNP. For simplicity, we only consider the case
when the blocks $x$ and $y$ are initially in \emph{different towers} and remains
so on the non-goal reachable states (i.e., there are no actions to put blocks
above $x$ or $y$ except \textit{Put-$x$-on-$y$} that achieves the goal).
In this case, the QNP $Q=\tup{F,V,I,O,G}$ has boolean and numerical variables
$F=\{E,X,D\}$ and $V=\{n,m\}$ that represent  whether the gripper is empty ($E$)
or holding the block $x$ ($X$), or if the goal has been achieved ($D$), while
the numerical variables $n$ and $m$ count the number of blocks above $x$ and $y$
respectively.
The initial state $I=\{E,\neg X,\neg D,\GT{n}, \GT{m}\}$ describes a configuration
where no block is being held, there are blocks above $x$ and above $y$,
but $x$ and $y$ are in different towers; the goal is simply $G = \{D\}$.
The QNP has six  different actions:\footnote{One reason for why
  this particular QNP is not suitable for dealing with states where
  the blocks $x$ and $y$ are in  the same tower, is that
  the $\textit{Pick-above-$x$}$ and  $\textit{Pick-above-$y$}$
  actions are not sound  relative to the intended  features $\phi_n$ and $\phi_m$ that the variables
  $n$ and $m$ are aimed to track. A needed $\textit{Pick-above-$x$-and-$y$}$ action,
  for example, will decrement  the two variables  $n$ and $m$.}
%
% from blocks_on.qnp
% Pick-x
% 2 empty 1 n 0
% 2 empty 0 holding-x 1
% Pick-above-x
% 2 empty 1 n 1
% 2 empty 0 n 0
% Pick-above-y
% 2 empty 1 m 1
% 2 empty 0 m 0
% Put-aside
% 2 empty 0 holding-x 0
% 1 empty 1
% Put-x-aside
% 2 empty 0 holding-x 1
% 2 empty 1 holding-x 0
% Put-x-on-y
% 3 empty 0 holding-x 1 m 0
% 4 empty 1 holding-x 0 goal 1 m 1
%
\begin{enumerate}[--]
  \item $\textit{Pick-$x$}       = \abst{E, \EQ{n}}{\neg E, X}$ to pick block $x$,
  \item $\textit{Pick-above-$x$} = \abst{E, \GT{n}}{\neg E, \DEC{n}}$ to pick the topmost block that is above $x$,
  \item $\textit{Pick-above-$y$} = \abst{E, \GT{m}}{\neg E, \DEC{m}}$ to pick the topmost block that is above $y$,
  \item $\textit{Putaside}       = \abst{\neg E, \neg X}{E}$ to put aside (not above $x$ or $y$) the block being held,
  \item $\textit{Put-x-aside}    = \abst{\neg E, X}{E, \neg X}$ to put aside the block $x$ (being held), and
  \item $\textit{Put-$x$-on-$y$} = \abst{\neg E, X, \EQ{m}}{E, \neg X, D, \INC{m}}$ to put $x$ on $y$.
\end{enumerate}

Since $Q$ has no action to increment $n$, \qnptofond generates in less than 0.01 seconds
a simplified FOND problem $P=T(Q)$ that has 47 atoms and 35 actions, 42 atoms for encoding
the counters and the stack, and 27 actions that manipulate the stack and the top counter.
%
%QNP features increments, the translator switches to the complete
%translation, but it notices that there are no increments for $n$ (the number
%of blocks above $x$), and thus the optimization for subset $\{n\}$ triggers.
%The translator runs in less than 0.01 seconds and generates a FOND problem
%$P=T(Q)$ with 56 atoms and 45 actions (51 atoms for encoding the counters and
%the stack, and 35 actions to manipulate the stack and move the top counter).
FOND-SAT finds the solution shown in Figure~\ref{fig:on} in 2.55 seconds;
it makes 9 calls to the SAT solver that require 0.31 seconds in total.
The resulting  controller also defines a  memoryless policy.

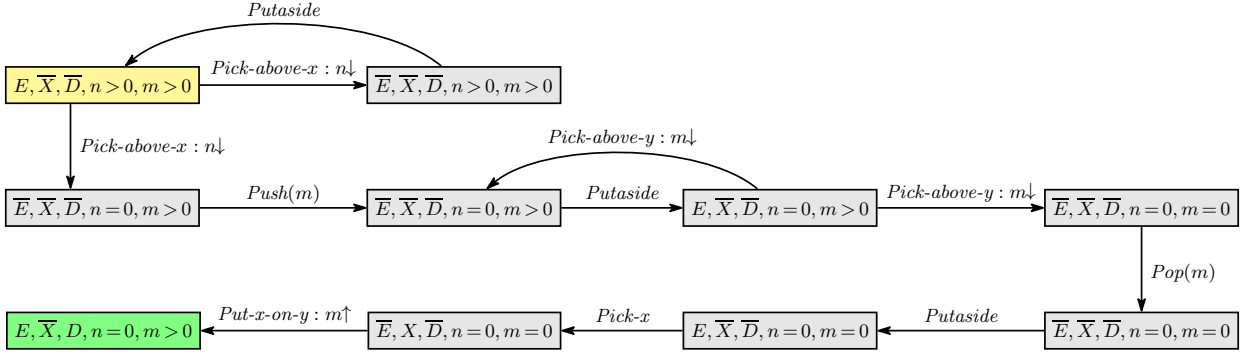
\begin{figure}[t]
  \centering
  \resizebox{\textwidth}{!}{
  \begin{tikzpicture}[thick,>={Stealth[inset=2pt,length=8pt,angle'=33,round]},font={\footnotesize},node distance=2cm,qs/.style={draw=black,fill=gray!20!white},init/.style={qs,fill=yellow!50!white},goal/.style={qs,fill=green!50!white}]%, show background rectangle, show background grid]
    \node[init]   (n0) { $E, \overline{X}, \overline{D}, \GT{n}, \GT{m}$ };
    \node[qs, right = 2.75cm of n0] (n2) { $\overline{E}, \overline{X}, \overline{D}, \GT{n}, \GT{m}$ };
    \node[qs, below = 1.4cm of n0] (n1) { $\overline{E}, \overline{X}, \overline{D}, \EQ{n}, \GT{m}$ };
    \node[qs, right =2.75cm of n1] (n3) { $\overline{E}, \overline{X}, \overline{D}, \EQ{n}, \GT{m}$ };
    \node[qs, right = of n3] (n4) { $E, \overline{X}, \overline{D}, \EQ{n}, \GT{m}$ };
    \node[qs, right = 2.75cm of n4] (n5) { $\overline{E}, \overline{X}, \overline{D}, \EQ{n}, \EQ{m}$ };
    \node[qs, below = 1.4cm of n5] (n6) { $\overline{E}, \overline{X}, \overline{D}, \EQ{n}, \EQ{m}$ };
    \node[qs, left = 2.75cm of n6]  (n7) { $E, \overline{X}, \overline{D}, \EQ{n}, \EQ{m}$ };
    \node[qs, left = of n7] (n8) { $\overline{E}, X, \overline{D}, \EQ{n}, \EQ{m}$ };
    \node[goal, left =2.75cm of n8] (ng) { $E, \overline{X}, D, \EQ{n}, \GT{m}$ };
    \path[->] (n0) edge[transform canvas={xshift=-15}] node[right,yshift=0] { $\textit{Pick-above-$x$}:\DEC{n}$ } (n1);
    \path[->] (n0) edge[] node[above,yshift=0] { $\textit{Pick-above-$x$}:\DEC{n}$ } (n2);
    \path[->] (n2) edge[out=140,in=40,looseness=0.7] node[above,yshift=0] { $\textit{Putaside}$ } (n0);
    \path[->] (n1) edge[] node[above,yshift=-2] { $\textit{Push}(m)$ } (n3);
    \path[->] (n3) edge[] node[above,yshift=0] { $\textit{Putaside}$ } (n4);
    \path[->] (n4) edge[] node[above,yshift=-1] { $\textit{Pick-above-y}:\DEC{m}$ } (n5);
    \path[->] (n4) edge[out=140,in=40,looseness=0.7] node[above,yshift=-1] { $\textit{Pick-above-y}:\DEC{m}$ } (n3);
    \path[->] (n5) edge[] node[right,yshift=-2] { $\textit{Pop}(m)$ } (n6);
    \path[->] (n6) edge[] node[above,yshift=0] { $\textit{Putaside}$ } (n7);
    \path[->] (n7) edge[] node[above,yshift=0] { $\textit{Pick-x}$ } (n8);
    \path[->] (n8) edge[] node[above,yshift=-1] { $\textit{Put-x-on-y}:\INC{m}$ } (ng);
  \end{tikzpicture}}
  \caption{%
    Solution of the FOND translation for the QNP $Q$ for placing a block $x$ on top of another block $y$.
    Nodes represent states in the translation (i.e., QNP states augmented with stack and counters) but only the QNP part is shown.
    Edges correspond to actions from $Q$ or actions that manipulate the stack and counters. Edges are annotated with action labels and their effect on the numerical variables. %(the effects on numerical variables are shown).
    The initial state is the top leftmost state and the goal is the leftmost one at the bottom. % (only one goal state shown; goal is $D$).
    The policy is strong cyclic and terminating.
  }
  \label{fig:on}
\end{figure}

\subsection{Gripper}

The task involves a robot with grippers whose goal is to move a number of balls from
one room into a target room.  Each gripper may carry one ball at a time.
An abstraction for generalized plan may be obtained with a QNP $Q=\tup{F,V,I,O,G}$
that involves one boolean feature $T$ that indicates whether the robot is in the target
room, and three numerical features that count the number of balls still to be moved ($b$),
the number of balls being carried ($c$), and the number of empty grippers ($g$).
The initial state $I=\{T,\GT{b},\EQ{c},\GT{g}\}$ places the robot at the target
room, carrying no balls, and with some balls in the other room. The goal description
is simply $G=\{\EQ{c},\EQ{b}\}$ saying that the number of balls being carried %by the robot
and the number of balls in the other room are both zero.
The set of (abstract) actions in $Q$ is:
%
% from gripper.qnp
% Drop-at-source
% 2 at-target 0 carrying 1
% 3 carrying 0 free-grippers 1 balls-at-source 1
% Drop-at-target
% 2 at-target 1 carrying 1
% 2 carrying 0 free-grippers 1
% Pick-at-source
% 3 at-target 0 balls-at-source 1 free-grippers 1
% 3 balls-at-source 0 free-grippers 0 carrying 1
% Move-to-target
% 1 at-target 0
% 1 at-target 1
% Move-back
% 1 at-target 1
% 1 at-target 0
%
\begin{enumerate}[--]
  \item $\textit{Drop-at-source} = \abst{\neg T, \GT{c}}{\INC{b},\DEC{c},\INC{g}}$ to drop balls in the other room,
  \item $\textit{Drop-at-target} = \abst{T, \GT{c}}{\DEC{c}, \INC{g}}$ to drop balls in the target room,
  \item $\textit{Pick-at-source} = \abst{\neg T, \GT{b}, \GT{g}}{\DEC{b}, \INC{c}, \DEC{g}}$ to pick balls in the other room,
  \item $\textit{Move}           = \abst{\neg T}{T}$ to move to the target room, and
  \item $\textit{Leave}          = \abst{T}{\neg T}$ to move to the other room.
\end{enumerate}
It is easy to see that this abstraction  captures
instances  involving any number of balls and robots with any positive number of grippers.
%The QNP abstraction suffices to yield a general policy.
%but is not a \emph{complete} abstraction \cite{bonet:ijcai2018}.
%There is indeed  no  action for picking balls at the target room, which would require an additional numerical
%feature tracking the number of balls in such room.

The translator runs in less than 0.01 seconds and generates a FOND problem
$P=T(Q)$ with 54 atoms and 47 actions (50 atoms for encoding the counters and
the stack, and 29 actions to manipulate the stack and move the top counter).
FOND-SAT finds the solution shown in Figure~\ref{fig:gripper} in 11.25 seconds;
it makes 10 calls to the SAT solver that require 3.00 seconds in total.
The resulting  controller also defines a  memoryless policy.

\begin{figure}[t]
  %\centering
  \resizebox{\textwidth}{!}{
  \begin{tikzpicture}[thick,>={Stealth[inset=2pt,length=8pt,angle'=33,round]},font={\footnotesize},node distance=2cm,qs/.style={draw=black,fill=gray!20!white},init/.style={qs,fill=yellow!50!white},goal/.style={qs,fill=green!50!white},qa/.style={qs,fill=cyan!50!white}]%, show background rectangle, show background grid]
    \useasboundingbox (-2.5,-5.15) rectangle (18.20, 1.25);

    % (n0,push_at_d0_b0,n1)
    % (n1,move-back,n2)
    % (n2,pick-at-source_f0_d1_DETDUP_3,n2)
    % (n2,pick-at-source_f0_d1_DETDUP_1,n4)
    % (n2,pick-at-source_f0_d1_DETDUP_0,n3)
    % (n2,pick-at-source_f0_d1_DETDUP_2,n3)
    % (n3,move-to-target,n5)
    % (n4,move-to-target,n6)
    % (n5,push_at_d1_b0,n7)
    % (n6,push_at_d1_b0,n8)
    % (n7,drop-at-target_f0_d2_DETDUP_0,ng)
    % (n7,drop-at-target_f0_d2_DETDUP_1,n7)
    % (n8,drop-at-target_f0_d2_DETDUP_1,n8)
    % (n8,drop-at-target_f0_d2_DETDUP_0,n9)
    % (n9,pop_at_d2,n1)
    \node[init]  (n0) { $T,\GT{b},\EQ{c},\GT{g}$ };
    \node[qs, right= of n0]  (n1) { $T,\GT{b},\EQ{c},\GT{g}$ };
    \node[qa, right= of n1]  (n2) { $\overline{T},\GT{b},\GE{c},\GT{g}$ };
    \node[qa, right=3.75cm of n2]  (n3) { $\overline{T},\EQ{b},\GT{c},\GE{g}$ };
    \node[qa, below= 1.4cm of n3]  (n5) { $T,\EQ{b},\GT{c},\GE{g}$ };
    \node[qa, below= 1.4cm of n5]  (n7) { $T,\EQ{b},\GT{c},\GE{g}$ };
    \node[goal, left=3.75cm of n7]  (ng) { $T,\EQ{b},\EQ{c},\GT{g}$ };
    \node[qs, below= 1.4cm of n2]  (n4) { $\overline{T},\GT{b},\GT{c},\EQ{g}$ };
    \node[qs,  left= of n4]  (n6) { $\overline{T},\GT{b},\GT{c},\EQ{g}$ };
    \node[qs,  left= of n6]  (n9) { $T,\GT{b},\EQ{c},\GT{g}$ };
    \node[qa, below= 1.4cm of n9]  (n8) { $T,\GT{b},\GT{c},\GE{g}$ };

    \path[->] (n0) edge[] node[above,yshift=-2] { $\textit{Push}(b)$ } (n1);
    \path[->] (n1) edge[] node[above,yshift=0] { $\textit{Leave}$ } (n2);
    \path[->] (n2) edge[] node[above,yshift=0,xshift=0] { $\textit{Pick-at-source}:\DEC{b},\INC{c},\DEC{g}$ } (n3);
    \path[->] (n3) edge[] node[right,yshift=0] { $\textit{Move}$ } (n5);
    \path[->] (n5) edge[] node[right,yshift=0] { $\textit{Push}(c)$ } (n7);
    \path[->] (n7) edge[] node[above,yshift=0] { $\textit{Drop-at-target}:\DEC{c},\INC{g}$ } (ng);
    \path[->] (n2) edge[] node[right,yshift=0] { $\textit{Pick-at-source}:\DEC{b},\INC{c},\DEC{g}$ } (n4);
    \path[->] (n4) edge[] node[above,yshift=-2] { $\textit{Push}(c)$ } (n6);
    \path[->] (n6) edge[transform canvas={xshift=10}] node[sloped,yshift=6,xshift=-4] { $\textit{Move}$ } (n8);
    \path[->] (n8) edge[transform canvas={xshift=20}] node[left,yshift=0] { $\textit{Drop-at-target}:\DEC{c},\INC{g}$ } (n9);
    \path[->] (n9) edge[transform canvas={xshift=10}] node[sloped,yshift=6,xshift=-4] { $\textit{Pop}(c)$ } (n1);

    \path[->] (n2) edge[out=140,in=40,looseness=4] node[above,yshift=0] { $\textit{Pick-at-source}:\DEC{b},\INC{c},\DEC{g}$ } (n2);
    \path[->] (n7) edge[out=220,in=320,looseness=4] node[below,yshift=2] { $\textit{Drop-at-target}:\DEC{c},\INC{g}$ } (n7);
    \path[->] (n8) edge[out=220,in=320,looseness=4] node[below,yshift=2] { $\textit{Drop-at-target}:\DEC{c},\INC{g}$ } (n8);
  \end{tikzpicture}}
  \caption{%
    Solution of the FOND translation for the QNP $Q$ for Gripper.
    Nodes represent states in the translation (i.e., QNP states augmented with stack and counters) but only the QNP part is shown.
    Blue nodes in the controller represent one or more QNP states; % (the FOND-SAT planner builds a controller that may feature such nodes);
    e.g., the node $\{\overline{T},\GT{b},\GE{c},\GT{g}\}$ in the first row represents the QNP states where $\{\overline{T},\GT{b},\GT{g}\}$ hold
    and there is no restriction on the value of $c$.
    Edges correspond to actions from $Q$ or actions that manipulate the stack and counters. Edges are annotated with action labels and their effect on the numerical variables. %(the effects on numerical variables are shown).
    The initial state is the top leftmost state and the goal is the second one at the bottom row. % (only one goal state shown; goal is $\EQ{b} \land \EQ{c}$).
    The policy is strong cyclic and terminating because each self loop is terminating, and the outer loop terminates since
    the variable $b$ decreases by at least one in each iteration, and no action in the plan increases it.
  }
  \label{fig:gripper}
\end{figure}

\subsection{Delivery}

The last example involves an agent that navigates a grid and whose job is to look for packages and deliver them at a target location
subject to the constraint that it can carry one package at a time.
The generalized problem consists of all instances with a finite but unbounded grid, and a finite but unbounded number of
packages. The generalized problem can be captured with the QNP $Q=\tup{F,V,I,O,G}$ that involves one boolean feature $H$
that tells whether the agent is holding a package, and 3 numerical features that measure the distance to the next package ($d$),
the distance to the target location ($t$), and the number of packages that still need to be delivered ($p$).
The initial state $I=\{\neg H,\GT{d},\GT{t},\GT{p}\}$ corresponds to a state where the agent holds no package and is neither
at a package or the target location, while the goal description $G=\{\neg H,\EQ{p}\}$ indicates that all packages have been delivered.
The QNP $Q$  has five actions:\footnote{From %
  the point of view of generalized planning,  the QNP  actions \textit{Move} and \textit{Home} are not  sound
  over this domain  since the agent may move towards the next package without moving away from the target location, and vice versa.}
%
% from delivery2.qnp
% Move-to-pkg
% 1 dist-to-next-pkg 1
% 2 dist-to-next-pkg 0 dist-to-target 1
% Move-to-target
% 1 dist-to-target 1
% 2 dist-to-target 0 dist-to-next-pkg 1
% Pick-pkg
% 2 dist-to-next-pkg 0 holding 0
% 1 holding 1
% Drop
% 2 holding 1 dist-to-target 1
% 1 holding 0
% Deliver
% 2 holding 1 dist-to-target 0
% 3 holding 0 pkgs 0 dist-to-next-pkg 1
%
% solving delivery2.qnp which has no Move-away-from-pkg nor Move-away-from-target
\begin{enumerate}[--]
  \item $\textit{Move}      = \abst{\GT{d}, \GT{p}}{\DEC{d}, \INC{t}}$ to move towards next package and away from target location,
  \item $\textit{Home}      = \abst{\GT{t}}{\INC{d}, \DEC{t}}$ to move towards target location and away from next package,
  \item $\textit{Pick}      = \abst{\neg H,\EQ{d}}{H}$ to pick a package,
  \item $\textit{Drop}      = \abst{H,\GT{t}}{\neg H}$ to drop a package not in target location, and
  \item $\textit{Deliver}   = \abst{H,\EQ{t},\GT{p}}{\neg H, \INC{d}, \DEC{p}}$ to deliver a package in target location.
\end{enumerate}

Since the variable $p$ is not incremented by any action, the translator generates
a simplified FOND problem $P=T(Q)$, in less than 0.01 seconds, that has 54 atoms and
40 actions (50 atoms for encoding the counters and the stack, and 29 actions to manipulate
the stack and move the top counter).
FOND-SAT finds the solution shown in Figure~\ref{fig:delivery} in 2.99 seconds;
it makes 8 calls to the SAT solver that require 0.30 seconds in total.
As in previous examples, the resulting controller defines a memoryless policy.

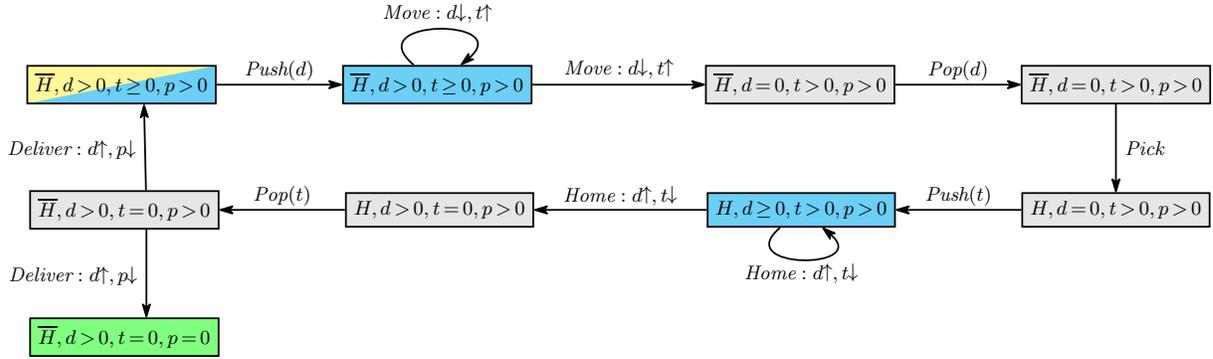
\begin{figure}[t]
  \centering
  \resizebox{\textwidth}{!}{
  \begin{tikzpicture}[thick,>={Stealth[inset=2pt,length=8pt,angle'=33,round]},font={\footnotesize},node distance=2cm,qs/.style={draw=black,fill=gray!20!white},init/.style={qs,fill=yellow!50!white},goal/.style={qs,fill=green!50!white},qa/.style={qs,fill=cyan!50!white}]%, show background rectangle, show background grid]
    %======== Actions with arguments ========
    %(n0,push_at_d0_b0)
    %(n0,push_at_d0_b0(dist-to-next-pkg))
    %(n1,move-to-pkg_f0_d1_DETDUP_1)
    %(n1,move-to-pkg_f0_d1_DETDUP_0())
    %(n1,move-to-pkg_f0_d1_DETDUP_1())
    %(n1,move-to-pkg_f0_d1_DETDUP_0)
    %(n2,pop_at_d1)
    %(n2,pop_at_d1(dist-to-next-pkg))
    %(n3,pick-pkg)
    %(n3,pick-pkg())
    %(n4,push_at_d0_b1)
    %(n4,push_at_d0_b1(dist-to-target))
    %(n5,move-to-target_f0_d1_DETDUP_0)
    %(n5,move-to-target_f0_d1_DETDUP_1())
    %(n5,move-to-target_f0_d1_DETDUP_1)
    %(n5,move-to-target_f0_d1_DETDUP_0())
    %(n6,pop_at_d1)
    %(n6,pop_at_d1(dist-to-target))
    %(n7,deliver_DETDUP_0())
    %(n7,deliver_DETDUP_1)
    %(n7,deliver_DETDUP_0)
    %(n7,deliver_DETDUP_1())
    %========= (CS, Action name, CS) ========
    %(n0,push_at_d0_b0,n1)
    %(n1,move-to-pkg_f0_d1_DETDUP_0,n2)
    %(n1,move-to-pkg_f0_d1_DETDUP_1,n1)
    %(n2,pop_at_d1,n3)
    %(n3,pick-pkg,n4)
    %(n4,push_at_d0_b1,n5)
    %(n5,move-to-target_f0_d1_DETDUP_1,n5)
    %(n5,move-to-target_f0_d1_DETDUP_0,n6)
    %(n6,pop_at_d1,n7)
    %(n7,deliver_DETDUP_1,n0)
    %(n7,deliver_DETDUP_0,ng)
    \node[init, diagonal fill={cyan!50!white}{yellow!50!white}] (n0) { $\overline{H},\GT{d},\GE{t},\GT{p}$ };
    \node[qa, right= of n0]       (n1) { $\overline{H},\GT{d},\GE{t},\GT{p}$ };
    \node[qs, right=2.75cm of n1]       (n2) { $\overline{H},\EQ{d},\GT{t},\GT{p}$ };
    \node[qs, right= of n2]       (n3) { $\overline{H},\EQ{d},\GT{t},\GT{p}$ };
    \node[qs, below=1.4cm of n3]  (n4) { $H,\EQ{d},\GT{t},\GT{p}$ };
    \node[qa, left= of n4]        (n5) { $H,\GE{d},\GT{t},\GT{p}$ };
    \node[qs, left=2.75cm of n5]        (n6) { $H,\GT{d},\EQ{t},\GT{p}$ };
    \node[qs, left= of n6]        (n7) { $\overline{H},\GT{d},\EQ{t},\GT{p}$ };
    \node[goal, below=1.4cm of n7](ng) { $\overline{H},\GT{d},\EQ{t},\EQ{p}$ };

    \path[->] (n0) edge[] node[above,yshift=-2] { $\textit{Push}(d)$ } (n1);
    \path[->] (n1) edge[] node[above,yshift=-1] { $\textit{Move}:\DEC{d},\INC{t}$ } (n2);
    \path[->] (n2) edge[] node[above,yshift=-2] { $\textit{Pop}(d)$ } (n3);
    \path[->] (n3) edge[] node[right,yshift=0] { $\textit{Pick}$ } (n4);
    \path[->] (n4) edge[] node[above,yshift=-2] { $\textit{Push}(t)$ } (n5);
    \path[->] (n5) edge[] node[above,yshift=-2] { $\textit{Home}:\INC{d},\DEC{t}$ } (n6);
    \path[->] (n6) edge[] node[above,yshift=-2] { $\textit{Pop}(t)$ } (n7);
    \path[->] (n7) edge[transform canvas={xshift=10}] node[left,yshift=-1] { $\textit{Deliver}:\INC{d},\DEC{p}$ } (n0);
    \path[->] (n7) edge[transform canvas={xshift=10}] node[left,yshift=-1] { $\textit{Deliver}:\INC{d},\DEC{p}$ } (ng);
    \path[->] (n1) edge[out=140,in=40,looseness=4] node[above,yshift=-1] { $\textit{Move}:\DEC{d},\INC{t}$ } (n1);
    \path[->] (n5) edge[out=220,in=320,looseness=4] node[below,yshift=2] { $\textit{Home}:\INC{d},\DEC{t}$ } (n5);
  \end{tikzpicture}}
  \caption{%
    Solution of the FOND translation for the QNP $Q$ for Delivery.
    Nodes represent states in the translation (i.e., QNP states augmented with stack and counters) but only the QNP part is shown.
    Blue nodes in the controller represent one or more QNP states; % (the FOND-SAT planner builds a controller that may feature such nodes);
    e.g., the top leftmost node $\{\overline{H},\GT{d},\GE{t},\GT{p}\}$ represents the QNP initial state and a similar state except $\EQ{t}$.
    Edges correspond to actions from $Q$ or actions that manipulate the stack and counters. Edges are annotated with action labels and their effect on numerical variables. %(the effects on numerical variables are shown).
    The initial state is the top leftmost state and the goal is the one at the bottom row. % (only one goal state shown; goal is $\neg H\land \EQ{p}$).
    The policy is strong cyclic and terminating because each self loop is terminating, and the outer loop terminates since
    the variable $p$ decreases by one in each iteration, and no action in the plan increases it.
  }
  \label{fig:delivery}
\end{figure}

\section{Related Work}

QNPs have been introduced as a decidable planning model able to account for plans with loops
\shortcite{sid:aaai2011,sid:aaai2015}. In addition, by defining the boolean and numerical variables
of QNPs as suitable general boolean and numerical features over a given domain, it has been
shown that QNPs can be used to express abstract models for generalized planning, in particular
when the ground actions change from instance to instance \shortcite{bonet:ijcai2018}.
More recently, it has been shown that these QNP abstractions can be learned automatically
from a given planning domain and sampled plans \shortcite{bonet:aaai2019}. QNPs thus provide
a convenient language for a \textbf{model-based approach} to  the computation of general
plans where such plans are derived from a (QNP) planning model. If the model is sound,
the general plans are guaranteed to be correct \shortcite{bonet:ijcai2018,bonet:aaai2019}.
This is in contrast with the more common \textbf{inductive} or \textbf{learning-based approaches}
where plans computed to solve a few sampled instances are assumed to generalize to
other instances by virtue of the compact form of the plans \shortcite{khardon:action,martin:applied,fern:generalized}.
These learning approaches do not construct or solve a suitable abstraction of the
problems as expressed by QNPs. Inductive approaches have been used recently to learn
general plans in the form of finite-state controllers \shortcite{bonet:icaps2009,hu:synthesis},
finite programs \shortcite{javier:procedures}, and deep neural nets learned in a supervised
manner \shortcite{trevizan:dl,sanner:dl,fern:dl,mausam:dl}. A key difference between learning-based
and model-based approaches is that the correctness of the latter follows from the soundness
of the model.
Deep reinforcement learning methods have also been used recently for computing generalized
plans with no supervision \shortcite{sid:sokoban,mazebase}, yet by not using first-order symbolic
representations, they have difficulties in dealing with relational domains that involve
objects and relations \shortcite{shanahan:review}. Forms of generalized planning have also been
formulated using first-order logic \shortcite{srivastava:generalized,sheila:generalized2019},
and general plans over finite horizons have been derived using first-order regression as
well \shortcite{boutilier2001symbolic,wang2008first,van2012solving,sanner:practicalMDPs}.
The use of QNPs for expressing (or learning) abstractions for generalized planning problems,
combined with the compilation of QNPs into FOND problems, allows us to benefit from the
performance of propositional off-the-shelf FOND planners like PRP \shortcite{prp}, MyND \shortcite{mynd},
and FOND-SAT \shortcite{geffner:fond-sat} in order to  find  generalized plans.

QNP problems can be easily translated into LTL planning problems with FOND domains,
reachability goals, and a particular type of trajectory constraints that can be expressed
as compact LTL formulas \shortcite{bonet:ijcai2017}. The trajectory constraints use a fragment of LTL
\shortcite{ltl} to express the QNP fairness constraints; namely, that \emph{trajectories where
a variable $X$ is decremented an infinite number of times and incremented a finite
number of times, are not fair} and can thus be ignored.\footnote{The LTL formula used by \citeay{bonet:ijcai2017}
  is slightly different but logically equivalent,   and states that always, if a variable is decreased infinitely often
  and increased finitely often, it must eventually have value zero. They are equivalent because
  once the value zero is reached, the variable cannot be decreased without being increased.}
As a result, QNP planning can be translated quite efficiently (linear time) into LTL synthesis.
The  translation, however, is not particularly useful computationally, as QNP planning,
like FOND planning, is EXP-Complete, while LTL synthesis is 2EXP-Complete (doubly exponential in time) \shortcite{ltl:2exp}.
In LTL planning, i.e., FOND planning with LTL goals and trajectory constraints, the double
exponential growth is in the number of variables that appear in such formulas
\shortcite{camacho:ltlplanning,sasha:ltlplanning}.
Tight complexity bounds for the specific type of LTL trajectory constraints that QNPs
convey have not been settled. %for LTL-based synthesis methods have not been settled.
In any case, such methods
need to compute in explicit form the QNP transition system, and thus require exponential
space in the total number of variables.
This lower bound, that does not consider the LTL formulas associated with the trajectory
constrains, already matches the upper bound of the brute-force algorithm that uses \sieve\ as subroutine
(cf.\ proof of Theorem~\ref{thm:plan-existence:expspace}).
%%The general method of \citeay{bonet:ijcai2017}, that can be applied to any transition
%%system, goal condition and trajectory constrains expressed in LTL, computes a tree
%%automaton that solves the given QNP (or proves that such automaton does not exist).
%%This method first computes a deterministic parity word (DPW) automaton that accepts the
%%models of an LTL formula that captures the QNP; this automaton may be of doubly exponential
%%size and with an exponential number of priorities for general types of LTL trajectory constraints.
%%%but it is ``only'' of exponential size and with a bounded number of priorities for QNPs.
%%Then, a deterministic parity tree automaton $A_t$, that accepts the policies
%%for the QNP and is built from the DPW automaton, must be tested for non-emptiness.
%%The tree automaton $A_t$ has size that is polynomial in the size of the DPW automaton
%%and with the same number of priorities.
%%The non-emptiness test requires time that is polynomial in the size of $A_t$
%%but exponential in the number of priorities.
%%%For QNPs, the number of priorities is bounded and thus this method can be
%%%implemented in exponential space since the DPW automaton must be built explicitly.
%%% Like the reduction from QNPs into FOND problems, this method does not solve the
%%% question posed above about the solvability of QNPs by memoryless policies
%%% since the automaton $A_t$ captures all history-based policies for the input QNP,
%%% not only memoryless policies.

\section{Conclusions}

QNPs are convenient abstract models for generalized planning.
In this work we have studied QNPs and placed them on firmer ground
by studying their theoretical foundations further.
We have also shown that FOND problems can be reduced into QNPs,
and vice versa, that QNPs can reduced into FOND problems.
Both translations are new and polynomial-time computable,
hence establishing that the two models have the same expressive power and the same complexity.
The previous, direct translation $T_D(Q)$ for QNPs $Q$ also yields a FOND problem but with
fairness assumptions that do not match those underlying strong cyclic FOND planning,
and this is why solutions to this translation need to be checked for termination.
QNPs can be reduced to LTL synthesis and planning but these are harder computational tasks.
In the future, it would be interesting to study more general types of fairness assumptions and
the fragments of LTL that can be handled efficiently with methods similar to the ones
developed for QNPs that are based on polynomial-time translations and off-the-shelf FOND planners.

\section*{Acknowledgements}

We thank the associate editor, Patrik Haslum, and the anonymous reviewers
for useful comments that helped us to improve the paper.
%The 3 small examples in Sect.~9 were indeed suggested by one of the reviewers in order to better illustrate the QNP to FOND reduction.
This work was performed while B.\ Bonet was at sabbatical leave at Universidad
Carlos III de Madrid under a UC3M-Santander C\'atedra de Excelencia Award.
H.\ Geffner is also a Guest WASP professor at Link\"oping University, Sweden,
and his work is partially funded by a grant TIN-2015-67959-P from MINECO, Spain,
and a grant from the Knut and Alice Wallenberg (KAW) Foundation, Sweden.

\bibliographystyle{theapa}
\bibliography{control}

\end{document}